\documentclass{article} 

\PassOptionsToPackage{dvipsnames,table}{xcolor}

\IfFileExists{iclr2026_conference.sty}{
  \usepackage{iclr2026_conference,times}
}{
  \usepackage{times}
}
\IfFileExists{math_commands.tex}{

\usepackage{amsmath,amsfonts,bm}

\def\eqref#1{equation~\ref{#1}}

\def\1{\bm{1}}

\def\eps{{\epsilon}}

\DeclareMathAlphabet{\mathsfit}{\encodingdefault}{\sfdefault}{m}{sl}
\SetMathAlphabet{\mathsfit}{bold}{\encodingdefault}{\sfdefault}{bx}{n}

}{}

\usepackage{booktabs}
\usepackage{tabularx}
\usepackage{pifont}
\usepackage{thmtools}
\newcommand{\cmark}{\textcolor{green!60!black}{\ding{51}}}
\newcommand{\xmark}{\textcolor{red!70!black}{\ding{55}}}
\usepackage[table]{xcolor}
  \definecolor{takeawayframe}{HTML}{2F5C8A}
  \definecolor{takeawayback}{HTML}{F2F6FB}

\usepackage{amsmath,amssymb,amsthm,mathtools,bm}
\usepackage{hyperref}
\usepackage{url}
\usepackage{titletoc}
\usepackage{enumitem}
\usepackage{xspace}
\usepackage{fontawesome5}
\makeatletter
\newcommand{\resetwrappar}{%
  \par
  \let\par\@@par
}
\makeatother
\usepackage{algorithm}
\usepackage{algpseudocode}

\usepackage{multirow}

\usepackage{wrapfig}

\usepackage[most]{tcolorbox}

\definecolor{targetblue}{RGB}{55,95,145}

\newtcolorbox{targetriskbox}{
    enhanced,
    breakable,
    colback=targetblue!4,
    colframe=targetblue!65,
    boxrule=0.6pt,
    arc=2pt,
    left=6pt,
    right=6pt,
    top=6pt,
    bottom=6pt,
    before skip=8pt,
    after skip=8pt,
    title={Target generative risk},
    fonttitle=\bfseries,
    coltitle=black
}

\definecolor{limitationred}{RGB}{250,65,15}
\newtcolorbox{limitationbox}{
    enhanced,
    breakable,
    colback=limitationred!4,
    colframe=limitationred!65,
    boxrule=0.6pt,
    arc=2pt,
    left=6pt,
    right=6pt,
    top=6pt,
    bottom=6pt,
    before skip=8pt,
    after skip=8pt,
    title={Limitations of single-adjustment learners},
    fonttitle=\bfseries,
    coltitle=black
}

\newcommand{\ind}{\mathbf 1}
\newcommand{\dd}{\,\mathrm d}
\def\eps{\varepsilon}
\newcommand{\cG}{\mathcal G}

\newcommand{\Pb}{\mathbb P_n}
\newcommand{\bxi}{\widehat \xi}
\newcommand{\bpi}{\widehat \pi}
\newcommand{\bmhat}{\widehat m}
\newcommand{\bR}{\widehat R}
\newcommand{\br}{\widehat r}
\newcommand{\barat}{\bar a_{t:t+\tau-1}}
\newcommand{\Ybarat}{Y_{t+\tau}[\barat]}

\newtheorem{assumption}{Assumption}
\newtheorem{proposition}{Proposition}
\newtheorem{remark}{Remark}

\newtheorem{lemma}{Lemma}
\newtheorem{corollary}{Corollary}

\usepackage{tikz}

\newcommand*\circledgreen[1]{%
\tikz[baseline=(char.base)]{
  \node[shape=circle, draw=ForestGreen!60, fill=ForestGreen!10, thick, inner sep=1pt] (char) {\scriptsize\textsf{#1}};
}}

\newcommand*\circledblue[1]{%
\tikz[baseline=(char.base)]{
  \node[shape=circle, draw=NavyBlue!60, fill=NavyBlue!10, thick, inner sep=1pt] (char) {\scriptsize\textsf{#1}};
}}
\usepackage{pifont}

\usepackage{etoolbox}

\title{\OrthoGen: A Generative Orthogonal Learner for Time-Varying Treatments}

\author{
\normalfont
\textbf{Tomàs Garriga}$^{1,2,3*}$, 
\textbf{Valentyn Melnychuk}$^{4,5}$,
\textbf{Konstantin Hess}$^{4,5}$, \\[1.5mm]
\textbf{Eduard Serrahima de Cambra}$^{1}$,
\textbf{Axel Brando}$^{2}$,
\textbf{Gerard Sanz}$^{1}$,
\textbf{Stefan Feuerriegel}$^{4,5}$
\\[2.5mm]
$^{1}$Novartis
\\
$^{2}$Barcelona Supercomputing Center
\\
$^{3}$Universitat Politècnica de Catalunya \\
$^{4}$LMU Munich
\\
$^{5}$Munich Center for Machine Learning
\\[1mm]
$^{*}$Corresponding author: 
\texttt{tomas.garriga\_dicuzzo@novartis.com}
}

\iclrfinalcopy 

\newcommand{\OrthoGen}{\textsc{OrthoGen}\xspace}

\begin{document}
\maketitle

\begin{abstract}
Estimating conditional distributional potential outcomes (CDPOs) \emph{over time} is important in medicine (e.g., to estimate patient-specific risks under different treatment sequences). However, this task is challenging because of time-varying confounding, yet existing adjustment strategies for this task are limited. In this paper, we aim to learn CDPOs under time-varying treatments using flexible generative models. Our contributions are two-fold. (1)~We introduce a tailored adjustment strategy for our setting, namely, \emph{generative recursive g-computation}. Our adjustment strategy recursively propagates full conditional outcome distributions rather than conditional means, modeling the variables of interest directly rather than full trajectories. Building on our adjustment strategy, we formulate simple generative learners for CDPO estimation. However, these learners can be sensitive to nuisance estimation errors, which motivates an orthogonal learner. (2)~We thus introduce \OrthoGen, a Neyman-orthogonal and doubly robust generative learner. Importantly, we show that \OrthoGen further achieves rate double robustness and quasi-oracle efficiency under suitable conditions. Our learners are flexible and can be instantiated with different generative backbones (e.g., normalizing flows and diffusion models). Across experiments with synthetic, semi-synthetic and real-world datasets, we find that \OrthoGen is highly effective. To the best of our knowledge, we are the first to propose a generative orthogonal learner for estimating CDPOs under time-varying treatments.

\end{abstract}

\section{Introduction}


Estimating the causal effects of time-varying treatments is important in medicine \citep{feuerriegel2024causal}. For example, in cancer care, chemotherapy regimens are updated across several treatment cycles \citep{press2016response}, while, in diabetes care, insulin doses are repeatedly adjusted in response to a patient's health condition \citep{meneghini2007usage}. Clinicians therefore need to understand how different treatment sequences affect the patient-level probability of future outcomes, such as tumor size, hypoglycemia, or toxicity. Here, \textbf{\textit{we aim to learn conditional distributional potential outcomes (CDPOs) in time-varying settings}}, which characterize the distribution of potential outcomes under a given treatment sequence, conditional on the observed patient history (see Fig. \ref{fig:distributions_example}).


A key challenge in this setting is \emph{time-varying confounding} \citep{robins2000marginal, frauen2025modelagnostic, hess2026igcnet}. Patient characteristics may change over time in response to earlier treatments and, in turn, affect both subsequent treatment decisions and future outcomes. As a result, adjusting only for the patient history observed at the prediction origin is generally insufficient; rather, \emph{adjustment strategies over time} 
are needed that account for the future treatment sequence. So far, the existing adjustment strategy for direct learning of CDPOs under time-varying treatments is inverse probability of treatment weighting (IPTW) \citep{mu2025counterfactual}. However, IPTW can suffer from unstable weights under limited overlap, especially at large horizons. Thus, the available adjustment strategies for direct generative CDPO estimation in time-varying settings remain limited, and we therefore develop alternative adjustment strategies. Our strategies avoid full simulation of covariate trajectories and model instead the outcomes of interest directly.





Existing methods have important gaps along three dimensions. \textbf{(1)}~\emph{Static vs. time-varying treatments.} Several generative methods learn CDPOs \citep{yoon2018ganite,vanderschueren2023noflite,ma2024diffpo,melnychuk2025gdr,luedtke2025doublegen}. However, these are restricted to static treatment and therefore do \textbf{not} address \emph{time-varying} treatment sequences. \textbf{(2)}~\emph{Mean outcome vs. CDPO.} Many methods for time-varying treatments learn conditional mean outcomes \citep[e.g.,][]{Bica2020Estimating,melnychuk2022causal,frauen2025modelagnostic,hess2024bayesian,hess2026igcnet}, but \textbf{not} the \emph{full} outcome distribution. \textbf{(3)}~\emph{Choice of adjustment strategy.} To the best of our knowledge, there is only one generative method for CDPOs under time-varying treatments: \citet{mu2025counterfactual} combine IPTW with an expert-guided diffusion model. However, IPTW can suffer from unstable weights and high variance \citep{frauen2025modelagnostic}. Thus, learners for CDPOs over time based on potentially more stable alternative adjustment strategies, such as g-computation, as well as learners offering favorable properties such as double robustness, are missing.



Here, \emph{we aim to learn CDPOs under time-varying treatments using flexible generative models}. To address this task, our contributions are two-fold.


\circledblue{1}\,We introduce a tailored adjustment strategy for our setting, which we call \emph{generative recursive g-computation}. Our approach extends standard iterative g-computation \citep{robins1986new} from conditional mean outcomes to \emph{full} conditional outcome distributions. The difference is that, instead of recursively propagating conditional expectations \citep{hess2026igcnet}, we recursively propagate full conditional outcome distributions through a sequence of generative models. Importantly, the recursive construction allows us to directly recover the target outcome distribution at the prediction origin, without simulating complete future covariate trajectories at inference time, which is particularly efficient for longitudinal settings. Building on our adjustment strategy, we develop simple generative learners for CDPO estimation, namely, plug-in and regression-adjusted learners. While straightforward, these learners can be sensitive to nuisance estimation errors, and we thus primarily use these learners to motivate an orthogonal learner ($\rightarrow$\,see \circledblue{2}).

\circledblue{2}\,We then develop \OrthoGen, a \textit{Neyman-orthogonal} and \textit{doubly robust} generative learner for CDPO estimation over time. For this, we combine our generative recursive g-computation with an additional propensity-based adjustment, which makes \OrthoGen first-order insensitive to nuisance estimation errors ($\rightarrow$ Theorem \ref{thm:longitudinal-orthogen-orthogonality}). \OrthoGen is flexible and can be instantiated with different generative backbones (e.g., normalizing flows and diffusion models). We further derive two favorable theoretical properties for \OrthoGen: (i)~\textit{quasi-oracle efficiency}, meaning that the target model can achieve the oracle learning rate even when the nuisance functions are estimated from observational data; and (ii)~\textit{rate double robustness}, meaning a slower convergence rate of one nuisance component can be compensated for by a faster convergence rate of the other, thus being robust to misspecification of either nuisance function ($\rightarrow$ Theorem \ref{thm:quasi-oracle}).

\begin{wrapfigure}{r}{0.65\textwidth}
    \centering
    \vspace{-0.3cm} \includegraphics[width=\linewidth]{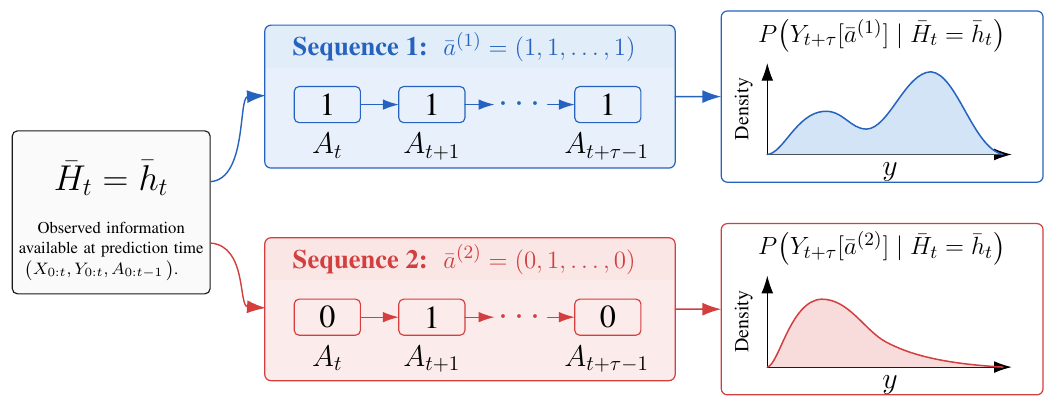}
    \vspace{-0.8cm} 
    \caption{\textbf{Task:} Estimating the CDPOs under different treatment sequences.}
    \vspace{-0.3cm} 

\label{fig:distributions_example}
\end{wrapfigure}


In summary, in this work:\footnote{Code is available at \url{https://github.com/tgarriga/OrthoGen}.}
\textbf{(1)}~We introduce generative recursive g-computation as a tailored adjustment strategy for CDPOs over time. \textbf{(2)}~We propose \OrthoGen, a doubly robust and Neyman-orthogonal generative learner for this task. We show theoretically that \OrthoGen achieves quasi-oracle efficiency and rate double robustness. \textbf{(3)}~We demonstrate the effectiveness of \OrthoGen using different generative backbones (i.e., normalizing flows and diffusion models) across various experiments. In particular, our experiments demonstrate the benefits of \OrthoGen over the IPTW-based learner.


\section{Related work\protect\footnote{We provide an extended related work in Appendix~\ref{app:extended-related-work}.}}

\textbf{(1)~Learning CDPOs in static vs. time-varying settings.} Several generative methods learn CDPOs in settings with a single static treatment, using different adjustment strategies and generative backbones \citep{yoon2018ganite,vanderschueren2023noflite,ma2024diffpo,luedtke2025doublegen}. This stream has recently culminated in a general class of learners for static settings called GDR-learners \citep{melnychuk2025gdr}. $\Rightarrow$ \emph{Our setting is different in that we consider a sequence of treatment decisions rather than a single treatment decision, which gives rise to \textbf{time-varying confounding} \citep{robins2000marginal, frauen2025modelagnostic} and therefore requires tailored adjustment strategies over time.}




\begin{wraptable}{r}{0.68\textwidth}
\vspace{-1.9\baselineskip}
\centering
\caption{Comparison of related approaches for direct learning of distributional potential outcomes
over time.}
\label{tab:related_work}

\scriptsize
\setlength{\tabcolsep}{1.5pt}
\renewcommand{\arraystretch}{1.05}

\begin{tabularx}{\linewidth}{
    >{\raggedright\arraybackslash}X
    >{\raggedright\arraybackslash}p{1.4cm}
    >{\centering\arraybackslash}p{1.85cm}
    >{\raggedright\arraybackslash}p{1.65cm}
    >{\centering\arraybackslash}p{1.75cm}
}
\toprule
\textbf{Method}
& \textbf{Adjustment}
& \textbf{Backbone}
& \textbf{Modeling target}
& \textbf{Neyman-orthogonality} \\
\midrule

\citet{wu2024counterfactual}
& IPTW
& CVAE / DM
& \xmark\ Marginal DPO
& \xmark \\

\citet{mu2025counterfactual}
& IPTW
& DM
& \cmark\ CDPO
& \xmark \\

\midrule

\textbf{PI} (\textit{ours})
& G.R. g-comp.
& Model agnostic
& \cmark\ CDPO
& \xmark \\

\textbf{RA} (\textit{ours})
& G.R. g-comp.
& Model agnostic
& \cmark\ CDPO
& \xmark \\

\textbf{\OrthoGen} (\textit{ours})
& DR
& Model agnostic
& \cmark\ CDPO
& \cmark \\

\bottomrule
\multicolumn{5}{@{}p{\dimexpr\linewidth-2\tabcolsep\relax}@{}}{
\scriptsize
Notes: G.R. g-comp. = generative recursive g-computation;
CVAE = conditional variational autoencoder; DM = diffusion model;
DPO = distributional potential outcome;
PI = plug-in;
RA = regression-adjusted;
DR = doubly robust.
}
\end{tabularx}
\vspace{-0.5\baselineskip}
\end{wraptable}

\textbf{(2)~Time-varying treatments with mean outcome vs. CDPO.}
A broad literature stream develops methods for causal estimation under time-varying treatments with different adjustment strategies but primarily targets the so-called conditional average potential outcomes (CAPOs), i.e., the conditional mean of the potential outcome under a given treatment sequence \citep[e.g.,][]{RMSN,Bica2020Estimating,melnychuk2022causal,frauen2025modelagnostic}. For CAPOs, there are some efforts to break the curse of dimensionality that is inherent to full simulation g-computation by identifying recursive procedures \citep{hess2026igcnet}, but these efforts are limited to CAPOs and are \emph{not} applicable to generative approaches.  $\Rightarrow$~\emph{The aforementioned methods target a different estimand: the conditional mean outcomes or treatment effects, rather than the \textbf{full} conditional distribution of potential outcomes as in our work.}



A different stream models the full longitudinal data-generating process (DGP) to enable interventional sampling by simulating complete trajectories \citep[e.g.,][]{gnet,pmlr-v252-xiong24a,pmlr-v259-deng25a,yeom2026estimating,wu2026doflow}. However, these works have a different modeling target: they aim to model and simulate the full DGP rather than directly learning the CDPO of the target outcome.

\textbf{(3)~Choice of adjustment strategies for CDPOs over time.} Only a few methods directly learn CDPOs under time-varying treatments with explicit confounding adjustment. \citet{wu2024counterfactual} use IPTW to learn the \emph{marginal} potential outcome distribution, which is a related but \emph{different} estimand. Closest to our work is \citet{mu2025counterfactual}, who combine IPTW with an expert-guided diffusion model to directly target CDPOs. However, IPTW can suffer from high variance and, most importantly, it is sensitive to nuisance misspecification. $\Rightarrow$~\emph{The only existing method relies on IPTW adjustment. In contrast, alternative and potentially more favorable adjustment strategies (e.g., g-computation) are missing for our task. As a result, generative learners with favorable learning properties (e.g., Neyman-orthogonality and double robustness) are also missing.}


\textbf{Research gap.} There are important gaps in the literature for estimating CDPOs through generative learners under time-varying treatments (see Table \ref{tab:related_work}). To the best of our knowledge, \OrthoGen is the first learner for this task that is Neyman-orthogonal and doubly robust.



\section{Problem setting}
\label{sec:problem-setting}

\textbf{Notation.}
Capital letters denote random variables and lowercase letters their realizations. We use $P(Z)$ to denote the distribution of a random variable $Z$, whereas $P(Z=z)$ denotes its density or probability mass function evaluated at $z$. Expectations are denoted by $\mathbb E[\cdot]$. For $n$ observations, we write $\Pb[f]=n^{-1}\sum_{i=1}^n f_i$ for the empirical average of a function $f$. When needed, $\mathbb P_0$ and $\mathbb E_0$ denote probability and expectation under the true observed-data distribution, and $\|f\|_{L^p(\mathbb P_0)}
= \left(\mathbb E_0\left[|f|^p\right]\right)^{1/p}$ denotes the corresponding $L^p$ norm. Estimated quantities are indicated by a hat, as in $\widehat f$. For time series, we write
$\bar U_t=U_{0:t}=(U_0,\ldots,U_t)$.

\textbf{Setup.}
We follow the standard setup for causal effect estimation with time-varying treatments \citep[e.g.,][]{Bica2020Estimating, melnychuk2022causal, frauen2025modelagnostic}. We observe $n$ i.i.d. longitudinal trajectories. For patient $i$, the observed trajectory contains outcomes $Y_t^{(i)}\in\mathbb R^{d_y}$, time-varying covariates $X_t^{(i)}\in\mathbb R^{d_x}$, and treatments $A_t^{(i)}\in\mathcal A$ over time steps $t\in\{0,\ldots,T\}$.
At time $t$, we write the observed pre-treatment history as $\bar H_t
=
\left(
\bar Y_t,\bar X_t,\bar A_{t-1}
\right)$.

\textbf{Estimand.} We fix a prediction origin $t$ and a prediction horizon $\tau\geq 1$. Within the intervention window, $\delta=0,\ldots,\tau-1$ denotes the relative time offset, and the corresponding absolute treatment time is $t+\delta$. We aim to learn the CDPO for $\tau$ time steps ahead given the interventional treatment sequence $\barat=(a_t,\ldots,a_{t+\tau-1})$. We make use of the potential outcomes framework for time-varying settings \citep{robins2000marginal, hernanrobins2025whatif}. Specifically, $Y_{t+\tau}[a_{t+\delta:t+\tau-1}]$ is the potential outcome at future time step $t+\tau$ under the remaining treatment sequence from relative offset $\delta$. Our aim is to learn the CDPO given by
\begin{equation}
\label{eq:target-law}
P\left( \Ybarat \mid \bar{H}_t=\bar{h}_t \right).
\end{equation}


\textbf{Identifiability.}
In order to ensure identifiability, we follow the standard assumptions for time-varying treatments \citep{robins1986new, hernanrobins2025whatif, frauen2025modelagnostic, hess2026igcnet}: (i)~\textit{Recursive consistency/composition:} for every $\delta=0,\ldots,\tau-1$ and every treatment sequence $a_{t+\delta:t+\tau-1}$ under consideration, if $A_{t+\delta}=a_{t+\delta}$, then
$Y_{t+\tau}[a_{t+\delta:t+\tau-1}]
=Y_{t+\tau}[a_{t+\delta+1:t+\tau-1}]$. In particular, if $A_{t:t+\tau-1}=a_{t:t+\tau-1}$, then $\Ybarat=Y_{t+\tau}$. (ii)~\textit{Positivity:} for every $\bar h_{t+\delta}$ in the support of $\bar H_{t+\delta}$, $P(A_{t+\delta}=a_{t+\delta}\mid\bar H_{t+\delta}=\bar h_{t+\delta})>0$, $\delta=0,\ldots,\tau-1$. (iii) \textit{Sequential exchangeability:} for every $\delta=0,\ldots,\tau-1$ and every treatment sequence $a_{t+\delta:t+\tau-1}$,
$A_{t+\delta}\perp Y_{t+\tau}[a_{t+\delta:t+\tau-1}]\mid\bar H_{t+\delta}$.

\textbf{The challenge of time-varying confounding.}
For a treatment sequence ($\tau>1$), potential outcomes estimation is challenging because confounders may arise \emph{after} the prediction origin $t$. A variable observed before a later treatment $A_{t+\delta}$, with $\delta>0$, may be affected by earlier treatments $A_t$, while also influencing both the assignment of $A_{t+\delta}$ and the terminal outcome $Y_{t+\tau}$ \citep{robins2000marginal, frauen2025modelagnostic, hess2026overlapweighted}. Then, adjusting only for the history $\bar H_t$ available at the prediction origin leaves later treatment decisions confounded. Sequential adjustment must instead account for the updated history $\bar H_{t+\delta}$ at each relative treatment step, even though these intermediate histories have not yet occurred when the counterfactual distribution is requested at time $t$.

Importantly, na{\i}ve history adjustment is biased: conditioning on $\bar H_t$ and selecting patients for which the observed future treatment sequence equals $\bar a$ generally does \textbf{not} recover the interventional distribution:
\begin{equation}
P\left( \Ybarat \mid \bar H_t=\bar h_t \right)
\neq
P\left( Y_{t+\tau} \mid \bar H_t=\bar h_t, A_{t:t+\tau-1}=a_{t:t+\tau-1} \right).
\label{eq:naive-history-adjustment}
\end{equation}
The left-hand side is our target, which does not coincide with the na{\"i}ve history adjustment on the right-hand side, which is thus \emph{biased}.


\begin{targetriskbox}
To estimate the CDPO under a treatment sequence $\bar a$ using a generative model, we seek the best projection of the true
counterfactual distribution $P(\Ybarat\mid\bar H_t=\bar h_t)$ in a predefined class of generative models $\cG$. We denote a candidate model by $g_{\bar a}\in\cG$ and define the score for an outcome $y$ conditional on the pre-interventional history $\bar h_t$ as
\begin{equation}
\ell_{g_{\bar a}}(y,\bar h_t)
=
\mathbb E_{Z\sim\varepsilon_z(\cdot\mid y,\bar h_t)}
\left[ \log g_{\bar a}(y,Z\mid\bar h_t) \right].
\label{eq:generative-score}
\end{equation}
This generic notation allows us to accommodate different generative backbones, such as likelihood-based normalizing flows and diffusion models, through their respective training objectives.\footnotemark


Accordingly, the optimal model within $\cG$ is the one that maximizes this oracle target risk
\begin{equation}
\mathcal L_{\bar a}(g_{\bar a})
=
\mathbb E
\left[
\ell_{g_{\bar a}}\left(\Ybarat,\bar H_t\right)
\right].
\label{eq:oracle-risk}
\end{equation}
That is, the model in $\cG$ that best approximates the true CDPO according to the chosen generative objective. 

\end{targetriskbox}\footnotetext{The above formulation assumes a score-maximization objective; it analogously applies to generative models formulated through loss minimization by replacing the corresponding maximization/minimization operators.}

This oracle risk provides a target that we use for all learners below. The challenge is that it depends on the unobserved counterfactual outcome $\Ybarat$ and therefore cannot be evaluated directly from observational data. 

\textbf{Outline.} \circledblue{1}\,Below, we first introduce \emph{generative recursive g-computation}, which is a tailored adjustment strategy for learning CDPOs under time-varying treatments (Sec.~\ref{sec:generative-recursive-gcomp}). \circledblue{2}\,Building on this adjustment strategy, we then introduce a plug-in (PI) learner and a regression-adjusted (RA) learner for CDPO estimation (Sec.~\ref{sec:single-learners}). \circledblue{3}\,Finally, we develop \OrthoGen and show that it is doubly robust and Neyman-orthogonal (Sec.~\ref{sec:orthogen}).

\section{Generative recursive g-computation}
\label{sec:generative-recursive-gcomp}

We first introduce \emph{generative recursive g-computation}, a tailored adjustment strategy for learning CDPOs over time. It builds on g-computation \citep{robins1986new}, which identifies interventional outcomes by integrating over the longitudinal data-generating process under a specified treatment sequence. We extend this principle from conditional mean outcomes to the \emph{full conditional outcome distribution} and derive a recursive generative formulation that propagates the required distributional information backward across treatment times. Rather than simulating complete future covariate trajectories, this recursion \textit{directly} targets the distribution relevant for the final outcome.


\textbf{Overview.} We develop \emph{generative recursive g-computation} in two steps. First, we define a backward recursion over conditional outcome distributions and show that it identifies the target CDPO at the prediction origin. Second, we show how this recursion can be fitted using generative models through backward pseudo-sample fitting. 




\textbf{\emph{Recursive} identification of the CDPO.} Our goal is to express the target CDPO at the prediction origin through a sequence of quantities that can be identified from the observed data distribution. We proceed backward from the final treatment decision: there, the distribution of the terminal outcome is directly observed conditional on the current history and treatment. We then recursively propagate this distribution backward across earlier treatment decisions by averaging over the subsequent histories that arise under the respective treatment (see Fig. \ref{fig:bg_gcomp}).

To formalize this recursion, we define $\{\xi_{t+\delta}^{\bar a}\}_{\delta=0}^{\tau-1}$, which we refer to as \emph{continuation densities}. Intuitively, $\xi_{t+\delta}^{\bar a}$ represents the distribution of the terminal outcome that remains to be propagated backward from relative treatment step $\delta$, incorporating the remaining treatment sequence $a_{t+\delta:t+\tau-1}$.


$\bullet$\,\emph{Recursion at the final treatment stage}. Let
\begin{align}
\xi_{t+\tau-1}^{\bar a}
(y\mid\bar h_{t+\tau-1})
=
P\!\left(
Y_{t+\tau}=y
\ \Big|\ 
\bar H_{t+\tau-1}=\bar h_{t+\tau-1},
A_{t+\tau-1}=a_{t+\tau-1}
\right).
\label{eq:terminal-stage-xi}
\end{align}
The terminal density can be estimated from the observed outcomes $Y_{t+\tau}$ of patients receiving $a_{t+\tau-1}$ at time $t+\tau-1$.

$\bullet$\,\emph{Recursive step from $t+\delta+1 \rightarrow t+\delta$}. Starting from the final treatment time, we move backward recursively. For $\delta=0,\ldots,\tau-2$, we define
\begin{align}
\xi_{t+\delta}^{\bar a}(y\mid\bar h_{t+\delta})
=
\mathbb E\left[
\xi_{t+\delta+1}^{\bar a}(y\mid\bar H_{t+\delta+1})
\ \Big|\ 
\bar H_{t+\delta}=\bar h_{t+\delta},
A_{t+\delta}=a_{t+\delta}
\right].
\label{eq:xi-recursion}
\end{align}
At each relative treatment step $\delta$, the successor density $\xi_{t+\delta+1}^{\bar a}$ already incorporates the future treatment values $a_{t+\delta+1},\ldots,a_{t+\tau-1}$. Hence, the recursion only requires conditioning on the current treatment $A_{t+\delta}=a_{t+\delta}$ and averages the successor distribution over the next history $\bar H_{t+\delta+1}$.


The above construction extends the recursive identification principle of longitudinal g-computation from conditional means \citep{hess2026igcnet} to full conditional outcome distributions. The following result shows that each continuation density identifies the distribution of the terminal potential outcome under the remaining treatment sequence and that $\xi_t^{\bar a}$ identifies the target CDPO in Eq.~\ref{eq:target-law}.

\begin{proposition}
\label{prop:recursive-gcomp-identification}
Under the causal assumptions from Sec.~\ref{sec:problem-setting}, for every $\delta=0,\ldots,\tau-1$, history $\bar h_{t+\delta}$, and outcome $y$, $\xi_{t+\delta}^{\bar a}(y\mid\bar h_{t+\delta}) = P\left( Y_{t+\tau}[a_{t+\delta:t+\tau-1}]=y \mid \bar H_{t+\delta}=\bar h_{t+\delta} \right)$. In particular, $\xi_t^{\bar a}(y\mid\bar h_t) = P\left( \Ybarat=y \mid \bar H_t=\bar h_t \right)$.
\end{proposition}
\vspace{-0.35cm}
\proof{ See Appendix ~\ref{app:generative-recursive-gcomp}.}

\textbf{Fitting the continuation densities.} 
The recursion in Eq. \ref{eq:xi-recursion} defines the continuation densities, but does not yet provide a way to estimate them with generative models. At the final treatment time, this is straightforward because the corresponding outcomes $Y_{t+\tau}$ are observed. At earlier treatment times, however, outcomes following the remaining treatment sequence are not observed and therefore cannot be used directly as training targets. We address this challenge recursively: after fitting $\widehat\xi_{t+\delta+1}^{\bar a}$, we draw pseudo-outcomes $\widetilde Y_{t+\delta+1}\sim\widehat\xi_{t+\delta+1}^{\bar a}(\cdot\mid\bar H_{t+\delta+1})$ and fit $\widehat\xi_{t+\delta}^{\bar a}(\cdot\mid\bar H_{t+\delta})$ to these pseudo-outcomes among observations with $A_{t+\delta}=a_{t+\delta}$, for $\delta=\tau-2,\ldots,0$. This procedure yields the sequence of estimated continuation densities and, in particular, $\widehat\xi_t^{\bar a}$ as an estimate of the target CDPO. We provide pseudocode in Appendix \ref{app:generative-recursive-gcomp-pseudocode}.


\begin{figure}[t]
    \centering
    \includegraphics[width=\linewidth]{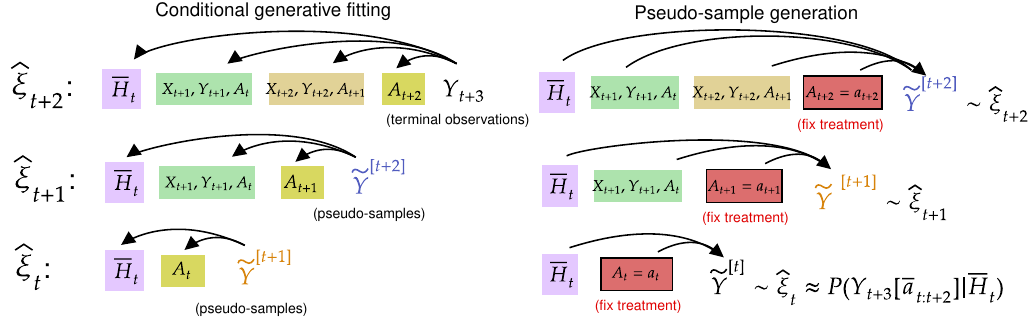}
    \vspace{-0.6cm}
    \caption{\textbf{Illustration of generative recursive g-computation} for $\tau = 3$.
\emph{Left:} Continuation densities are fitted through backward recursion. The terminal density is fitted from observed outcomes among patients receiving the target treatment. At each earlier treatment step $\delta$, pseudo-outcomes are sampled from the fitted successor density and used to fit the preceding density. \emph{Right:} Generating the pseudo-samples that will be used to fit the continuation nuisances at earlier steps.}
    \label{fig:bg_gcomp}
\end{figure}

\section{Na{\"i}ve generative learners based on recursive g-computation} 
\label{sec:single-learners}

We build directly on \emph{generative recursive g-computation} and introduce two generative learners that use its estimated continuation densities in different ways: The \textbf{plug-in (PI) learner} directly uses the fitted continuation density at the prediction origin as an estimate of the CDPO, while the \textbf{regression-adjusted (RA) learner} projects this estimate onto a separate target generative model through the corresponding generative risk.



$\bullet$\,\textbf{Plug-in (PI) learner.}
Proposition~\ref{prop:recursive-gcomp-identification} shows that the population continuation density at time $t$ equals the target CDPO in Eq.~\ref{eq:target-law}. Therefore, the fitted continuation density obtained through generative recursive g-computation at the prediction origin can be used as an estimate of the CDPO. The plug-in (PI) learner is
\begin{equation}
\widehat P_{\mathrm{PI}}^{\bar a}(y\mid \bar h_t)
=
\widehat\xi_t^{\bar a}(y\mid \bar h_t).
\end{equation}

$\bullet$\,\textbf{Regression-adjusted (RA) learner.} The regression-adjusted (RA) learner uses the estimated continuation density to construct a generative learning objective over the model class $\cG$. It follows a two-stage procedure: in Stage~1, generative recursive g-computation is used to estimate the continuation densities; in Stage~2, these estimates define a generative learning objective for fitting a model in $\cG$. This allows the second-stage model to incorporate, for example, additional regularization, architectural constraints, or interpretability requirements.

To construct this objective, for any candidate model $g_{\bar a}\in\cG$ and relative treatment step $\delta$, we define the \emph{continuation score}
\begin{equation}
m_{t+\delta}^{g_{\bar a},\bar a}(\bar H_{t+\delta}) = \int_{\mathcal Y} \ell_{g_{\bar a}}(y,\bar H_t)\, \xi_{t+\delta}^{\bar a}(y\mid\bar H_{t+\delta})\, \dd y, \qquad \delta=0,\ldots,\tau-1,
\label{eq:ra-score-response}
\end{equation}
which is the expected generative score of $g_{\bar a}$ under the continuation density $\xi_{t+\delta}^{\bar a}$. Further, let $\widehat m_{t+\delta}^{g_{\bar a},\bar a}$ denote the corresponding quantity obtained by replacing $\xi_{t+\delta}^{\bar a}$ with $\widehat\xi_{t+\delta}^{\bar a}$.

At the prediction origin, Proposition~\ref{prop:recursive-gcomp-identification} implies that $\xi_t^{\bar a}$ equals the target CDPO. Hence, the oracle target risk in Eq.~\ref{eq:oracle-risk} can equivalently be written as $\mathcal L_{\bar a}(g_{\bar a}) = \mathbb E\left[ m_t^{g_{\bar a},\bar a}(\bar H_t) \right]$. The previous identity suggests replacing the unknown continuation density with the corresponding estimate obtained via generative recursive g-computation. The RA learner is defined by the risk
\begin{equation}
\widehat{\mathcal L}_{\mathrm{RA}}^{\bar a}(g_{\bar a}) = \Pb\left[ \widehat m_t^{g_{\bar a},\bar a}(\bar H_t) \right].
\label{eq:ra-risk}
\end{equation}

Intuitively, the RA learner optimizes the estimated risk over $g_{\bar a}\in\cG$. Thus, the key difference is that the PI learner directly returns $\widehat\xi_t^{\bar a}$, whereas the RA learner uses $\widehat\xi_t^{\bar a}$ to define the learning objective for a second-stage generative model in $\cG$.

Thus, our PI and RA learners provide an alternative to the IPTW-based learner, which has been proposed for CDPO estimation over time \citep{mu2025counterfactual} and which we use as another baseline (see Appendix~\ref{app:iptw-learner}).

An important limitation of the two learners above and IPTW is that they rely on a single nuisance component, which makes the target risk sensitive to nuisance-estimation error at first order \citep{vansteelandt2025orthogonal}. To eliminate the first-order bias, we develop a Neyman-orthogonal method called \OrthoGen in the following section.


\section{\OrthoGen}
\label{sec:orthogen}



We introduce \OrthoGen for CDPO estimation under time-varying treatments, which, as we show later, enjoys several favorable theoretical properties (e.g., Neyman-orthogonality, double robustness, quasi-oracle efficiency). First, we define the propensity score
$\pi_{t+\delta}(a_{t+\delta}\mid\bar H_{t+\delta})=P(A_{t+\delta}=a_{t+\delta}\mid\bar H_{t+\delta})$, for $\delta=0,\ldots,\tau-1$. We then define the one-step inverse-propensity ratio as
$r_{t+\delta}^{\bar a}
= \ind\{A_{t+\delta}=a_{t+\delta}\}/
\pi_{t+\delta}(a_{t+\delta}\mid\bar H_{t+\delta})$
and the corresponding cumulative ratio as
$R_{t:t+\delta}^{\bar a}
= \prod_{\delta'=0}^{\delta} r_{t+\delta'}^{\bar a}$,
with the empty-product convention $R_{t:t-1}^{\bar a}=1$.


\begin{tcolorbox}[
  colback=takeawayback,
  colframe=takeawayframe,
  boxrule=0.6pt,
  arc=1pt,
  left=5pt,right=5pt,top=4pt,bottom=4pt,
  boxsep=0pt,
  before skip=5pt,
  after skip=5pt
]

\OrthoGen is defined by the risk:
\begin{align}
\begin{aligned}
\widehat{\mathcal L}_{\mathrm{OrthoGen}}^{\bar a}(g_{\bar a})
=
\Pb\Bigg[
\bR_{t:t+\tau-1}^{\bar a}
\ell_{g_{\bar a}}(Y_{t+\tau},\bar H_t) +
\sum_{\delta=0}^{\tau-1} \bR_{t:t+\delta-1}^{\bar a}
\left(1-\br_{t+\delta}^{\bar a}\right)
\bmhat_{t+\delta}^{g_{\bar a},\bar a}(\bar H_{t+\delta})
\Bigg].
\end{aligned}
\label{eq:orthogen-augmentation}
\end{align}
\end{tcolorbox}

\textbf{Estimation.} 
\OrthoGen proceeds in two stages: First, we estimate the nuisance functions $\hat{\eta}=
\left(
\{\hat{\pi}_{t+\delta}\}_{\delta=0}^{\tau-1},
\{\hat{\xi}_{t+\delta}^{\bar a}\}_{\delta=0}^{\tau-1}
\right)$. Second, we fit the target generative model using the objective in Eq.~\ref{eq:orthogen-augmentation}. The estimated propensity scores yield the one-step and cumulative inverse propensity ratios $\widehat r_{t+\delta}^{\bar a}$ and $\widehat R_{t:t+\delta}^{\bar a}$ defined above, while the estimated continuation densities determine the continuation scores $\widehat m_{t+\delta}^{g_{\bar a},\bar a}$ through Eq.~\ref{eq:ra-score-response}. \OrthoGen is model agnostic and we instantiate it with two generative models: normalizing flows and diffusion models.

\subsection{Theoretical properties}

\OrthoGen has several favorable asymptotic properties. First, it is Neyman-orthogonal, that is,
first-order insensitive to the errors in the nuisance functions. The following theorem derives this property for \OrthoGen:

\begin{restatable}[Neyman-orthogonality]{theorem}{orthogentheorem}
\label{thm:longitudinal-orthogen-orthogonality}
Under the causal assumptions in Sec.~\ref{sec:problem-setting}, the objective is Neyman-orthogonal at $(g^\star_{\bar a},\eta)$, namely
\begin{equation}
D_\eta D_{g_{\bar a}}
\mathcal L_{\mathrm{OrthoGen}}^{\bar a}
(g^\star_{\bar a};\eta)
\big[
g_{\bar a}-g^\star_{\bar a},
\tilde\eta-\eta
\big]
=0,
\end{equation}
for every \(g_{\bar a}\in\mathcal G\) and every admissible nuisance perturbation \(\tilde\eta\), where \(D\mathcal{L}(\cdot)[\cdot]\) denotes the pathwise derivative (see Appendix \ref{app:neyman-orthogonality} for definitions).
\end{restatable}
\vspace{-0.2cm}
\textit{Proof.} See Appendix~\ref{app:neyman-orthogonality}.

Second, \OrthoGen is quasi-oracle efficient and doubly robust. To formalize these properties, we define the fitted second-stage target generator as $\widehat g_{\bar a}\in\mathcal G$ and its conditional population target-learning error as
$\mathrm{Opt}^{\bar a}_{\mathcal G}
:=
\sup_{g_{\bar a}\in\mathcal G}
\mathcal L_{\mathrm{OrthoGen}}^{\bar a}(g_{\bar a};\widehat\eta)
-
\mathcal L_{\mathrm{OrthoGen}}^{\bar a}(\widehat g_{\bar a};\widehat\eta)$.
Thus, $\mathrm{Opt}^{\bar a}_{\mathcal G}$ measures how well the second-stage target generator is learned while holding the estimated nuisance functions fixed. 

\begin{restatable}[Quasi-oracle efficiency and double robustness]{theorem}{quasioracletheorem}
\label{thm:quasi-oracle}
Under mild regularity conditions for the target model, the squared distance between the fitted target generator $\widehat g_{\bar a}$ and the oracle population maximizer $g^\star_{\bar a}$ can be upper-bounded in the following way:  
\begin{equation}
\|\widehat g_{\bar a}-g^\star_{\bar a}\|_{\mathcal G}^2
\lesssim
\underbrace{
\mathrm{Opt}^{\bar a}_{\mathcal G}
}_{\textup{(I) target-model learning error}}
+
\underbrace{
\left[
\sum_{\delta=0}^{\tau-1}
\|\widehat\pi_{t+\delta}-\pi_{t+\delta}\|_{L^4(\mathbb P_0)}
\,
\|\widehat\xi_{t+\delta}^{\bar a}-\xi_{t+\delta}^{\bar a}\|_{4,2}
\right]^2
}_{\textup{(II) higher-order nuisance error}},
\label{eq:quasi_oracle_main}
\end{equation}
where $\|\cdot\|_{\mathcal G}$ denotes a model-specific norm, (I) is a standard target-model learning error term, and (II) is a higher-order nuisance error term. This bound implies that \OrthoGen is (a) \textbf{doubly robust} and (b) \textbf{quasi-oracle efficient}.
\end{restatable}

\vspace{-0.2cm}\textit{Proof.} See Appendix \ref{sec:double_robustness_quasi_oracle}.

Theorem~\ref{thm:quasi-oracle} decomposes the target-model error into a second-stage learning error and a higher-order contribution from nuisance estimation. Crucially, the latter depends on products of the stagewise propensity and continuation density errors rather than on either nuisance error alone. This product structure yields the following rate-robustness property.

\begin{remark}[Rate double robustness]
Let $\rho_n$ be the learning rate of the target model, and let $\varepsilon_{\pi,t+\delta,n}$ and $\varepsilon_{\xi,t+\delta,n}$ be the convergence rates of the propensity and continuation density nuisances at relative treatment step $\delta$, respectively. Because nuisance estimation enters through $\varepsilon_{\pi,t+\delta,n}\varepsilon_{\xi,t+\delta,n}$, faster convergence of one nuisance component can offset slower convergence of the other. In particular, if $\sum_{\delta=0}^{\tau-1}\varepsilon_{\pi,t+\delta,n}\varepsilon_{\xi,t+\delta,n}=o(\rho_n)$, then nuisance estimation is asymptotically negligible and \OrthoGen attains the same first-order rate as an oracle learner with known nuisances. More generally, if the product term is $O(\rho_n)$, nuisance estimation does not worsen the learning rate for the target model; for example, two $n^{-1/4}$ nuisance rates yield an $n^{-1/2}$ product term.
\end{remark}


\section{Evaluation}

In this section, we evaluate our learners through several numerical experiments across synthetic, semi-synthetic, and real-world datasets to assess their effectiveness.

\resetwrappar
\begin{wraptable}{r}{0.65\textwidth}
\vspace{-0.8cm}
\centering
\caption{\textbf{Synthetic data.}}
\label{tab:wasserstein_tau_synthetic}
\scriptsize
\setlength{\tabcolsep}{2pt}
\renewcommand{\arraystretch}{0.9}

\resizebox{\linewidth}{!}{%
\begin{tabular}{llccccc}
\toprule
& \textbf{Method} 
& $\boldsymbol{\tau=1}$ 
& $\boldsymbol{\tau=2}$ 
& $\boldsymbol{\tau=3}$ 
& $\boldsymbol{\tau=4}$ 
& $\boldsymbol{\tau=5}$ \\
\midrule
\multirow{4}{*}{NFs}
& PI           & $\mathbf{0.308 \pm 0.004}$ & $0.235 \pm 0.007$ & $0.350 \pm 0.015$ & $0.249 \pm 0.010$ & $\mathbf{0.344 \pm 0.015}$ \\
& RA           & $0.320 \pm 0.008$ & $0.232 \pm 0.008$ & $0.347 \pm 0.015$ & $\mathbf{0.247 \pm 0.008}$ & $0.348 \pm 0.018$ \\
& IPTW        & $0.325 \pm 0.009$ & $0.240 \pm 0.002$ & $0.373 \pm 0.015$ & $0.305 \pm 0.019$ & $0.518 \pm 0.062$ \\
& \OrthoGen    & $0.312 \pm 0.005$ & $\mathbf{0.229 \pm 0.009}$ & $\mathbf{0.342 \pm 0.011}$ & $0.267 \pm 0.015$ & $0.401 \pm 0.051$ \\
\midrule
\multirow{4}{*}{DMs}
& PI           & $\mathbf{0.306 \pm 0.013}$ & $0.312 \pm 0.025$ & $0.506 \pm 0.038$ & $0.482 \pm 0.015$ & $0.612 \pm 0.014$ \\
& RA           & $0.411 \pm 0.030$ & $0.404 \pm 0.018$ & $0.610 \pm 0.032$ & $0.537 \pm 0.025$ & $0.651 \pm 0.019$ \\
& IPTW        & $0.341 \pm 0.026$ & $0.284 \pm 0.020$ & $0.447 \pm 0.053$ & $0.386 \pm 0.048$ & $0.522 \pm 0.055$ \\
& \OrthoGen    & $0.329 \pm 0.014$ & $\mathbf{0.244 \pm 0.012}$ & $\mathbf{0.378 \pm 0.017}$ & $\mathbf{0.293 \pm 0.026}$ & $\mathbf{0.383 \pm 0.024}$ \\
\bottomrule
\end{tabular}%
}
\par\vspace{1pt}
\begin{minipage}{\linewidth}
\raggedright\tiny
Reported: squared 2-Wasserstein costs across treatment horizons $\tau$.\par
Values are mean $\pm$ standard deviation over 5 replicates; lower is better.
\vspace{-.1cm}
\end{minipage}
\end{wraptable}

\textbf{Instantiations.}  We instantiate the sequence encoder using a transformer \citep{vaswani2017attention}, which processes the observed pre-interventional history to produce a low-dimensional context representation fed to the nuisance and target models. For the generative models, we employ two complementary state-of-the-art approaches: \textbf{(i)}~\textit{normalizing flows} (\textbf{NFs}) \citep{rezende2015variational} and \textbf{(ii)}~\textit{diffusion models} (\textbf{DMs}) \citep{ho2020denoising}. See Appendix \ref{app:generative-model-instantiations} for more details. In both cases, we use the same generative backbone for both the continuation density nuisances and the target generative model. 
Our transformer implementation follows an encoder-only design similar to \citet{frauen2025modelagnostic} and \citet{hess2025stabilized}. Both generative model families receive the learned context representation and generate samples conditionally on the learned history representation. Details are in Appendix \ref{app:implementation-details}.

\textbf{Datasets.}
Following common practice in causal inference \citep{Bica2020Estimating, melnychuk2022causal}, we use synthetic and semi-synthetic datasets for evaluation as they provide ground truths for counterfactuals, while we use a real-world dataset to assess the applicability of our learners and discuss potential insights. We thus evaluate on three datasets; see Appendix~\ref{app:dataset-details} for full details. \circledgreen{1} A \textbf{fully synthetic benchmark} generated from a longitudinal data-generating process with treatment--confounder feedback, a scalar time-varying covariate, and a three-dimensional continuous target outcome, where the components are mixture distributions based on Gaussian and asymmetric-Laplace components and are correlated with each other. \circledgreen{2} A \textbf{semi-synthetic benchmark} built from real patient trajectories extracted from MIMIC-III \citep{johnson2016mimic}. \circledgreen{3} The \textbf{real-world MIMIC-III} benchmark \citep{johnson2016mimic} based on previous cohort construction and covariates, but where we retain the factual treatments and a one-dimensional observed outcome from the ICU records. All three datasets use 3,000 training, 1,000 validation, and 1,000 test trajectories.

\textbf{Training.} Our implementation follows a two-stage procedure. First, the propensity scores are estimated using a binary classifier predicting treatment assignment at each decision point from the transformer representation, while the continuation densities are learned through generative recursive g-computation. Hence, the terminal model is fitted to observed factual outcomes, and earlier continuation models are subsequently trained on pseudo-outcomes sampled from frozen successor models using Monte Carlo sampling to estimate expectations. Then, the target generator of \OrthoGen is fitted using the objective in Eq.~\ref{eq:orthogen-augmentation}.
See further details on architectures and hyperparameter selection in Appendix \ref{app:transformer-history-encoder} and \ref{app:generative-model-implementation} and in the code. We provide runtime information in Appendix~\ref{app:runtime}.

\begin{wraptable}{l}{0.60\textwidth}
\vspace{-1.8em}
\centering
\caption{\textbf{Semi-synthetic data.}}
\label{tab:wasserstein_tau_semisynthetic}

\scriptsize
\setlength{\tabcolsep}{2pt}
\renewcommand{\arraystretch}{0.9}

\resizebox{\linewidth}{!}{%
\begin{tabular}{llccccc}
\toprule
& \textbf{Method} 
& $\boldsymbol{\tau=1}$ 
& $\boldsymbol{\tau=2}$ 
& $\boldsymbol{\tau=3}$ 
& $\boldsymbol{\tau=4}$ 
& $\boldsymbol{\tau=5}$ \\
\midrule
\multirow{4}{*}{NFs}
& PI           & $1.13 \pm 0.39$ & $1.46 \pm 0.70$ & $1.59 \pm 0.21$ & $1.78 \pm 0.67$ & $2.55 \pm 0.58$ \\
& RA           & $1.22 \pm 0.33$ & $1.22 \pm 0.55$ & $1.47 \pm 0.23$ & $1.59 \pm 0.44$ & $2.22 \pm 0.52$ \\
& IPTW        & $0.78 \pm 0.11$ & $0.67 \pm 0.20$ & $1.45 \pm 0.57$ & $0.92 \pm 0.22$ & $2.06 \pm 0.87$ \\
& \OrthoGen    & $\mathbf{0.49 \pm 0.09}$ & $\mathbf{0.47 \pm 0.13}$ & $\mathbf{0.91 \pm 0.13}$ & $\mathbf{0.73 \pm 0.22}$ & $\mathbf{1.23 \pm 0.45}$ \\
\midrule
\multirow{4}{*}{DMs}
& PI           & $0.56 \pm 0.06$ & $0.59 \pm 0.13$ & $0.81 \pm 0.32$ & $0.93 \pm 0.36$ & $0.93 \pm 0.36$ \\
& RA           & $0.63 \pm 0.08$ & $0.74 \pm 0.20$ & $0.99 \pm 0.38$ & $1.09 \pm 0.39$ & $1.09 \pm 0.41$ \\
& IPTW        & $0.47 \pm 0.10$ & $0.62 \pm 0.27$ & $1.43 \pm 0.68$ & $1.77 \pm 0.79$ & $2.49 \pm 0.50$ \\
& \OrthoGen    & $\mathbf{0.45 \pm 0.05}$ & $\mathbf{0.50 \pm 0.10}$ & $\mathbf{0.74 \pm 0.21}$ & $\mathbf{0.88 \pm 0.30}$ & $\mathbf{0.77 \pm 0.18}$ \\
\bottomrule
\end{tabular}%
}
\par\vspace{1pt}
\begin{minipage}{\linewidth}
\raggedright\tiny
Reported: squared 2-Wasserstein costs $\times 10^{-2}$ across treatment horizons $\tau$.\par
Values are mean $\pm$ standard deviation over 5 replicates; lower is better.
\vspace{-0.9em}

\end{minipage}
\end{wraptable}

\textbf{Baselines.} To the best of our knowledge, the \textbf{IPTW} from \citet{mu2025counterfactual} is the \textbf{only} applicable baseline for direct learning of CDPOs. Other methods solve \textit{different} estimands (e.g., CAPOs but \emph{not} CDPOs, so that the full distributional information is lost). We thus additionally benchmark \OrthoGen against two na{\"i}ve learners: \textbf{PI} and \textbf{RA} from Sec.~\ref{sec:single-learners}. As a result, we can isolate the benefit of the doubly robust and Neyman-orthogonal adjustment approach in \OrthoGen. Importantly, all baselines are instantiated with the \emph{same} transformer architecture and generative backbones to ensure a \textit{fair} comparison. 


\textbf{Metrics.} For the \circledgreen{1} fully synthetic and \circledgreen{2} semi-synthetic datasets, we can obtain ground-truth interventional outcome samples from the simulation: for each test sample, we generate samples from the fitted model and matched oracle samples under the target regime. We compare the resulting predictive distributions using the squared 2-Wasserstein cost, which measures discrepancy between empirical outcome distributions. \circledgreen{3} For the real-world MIMIC-III dataset, no oracle counterfactual distribution is available. We thus assess factual predictive accuracy using the continuous ranked probability score (CRPS), which is a proper scoring rule for one-dimensional predictive distributions that requires only the observed factual outcome. Predictions are made under the target treatment sequence, and CRPS is computed only for held-out patients whose observed treatment sequence exactly matches the target. See Appendix~\ref{app:dataset-details} for details.

\begin{wraptable}{r}
{0.48\textwidth}
\vspace{-1.8em}
\centering
\caption{\textbf{Real-world data.}}
\label{tab:crps_tau_real}

\scriptsize
\setlength{\tabcolsep}{2pt}
\renewcommand{\arraystretch}{0.9}

\resizebox{\linewidth}{!}{%
\begin{tabular}{llccc}
\toprule
& \textbf{Method} 
& $\boldsymbol{\tau=1}$ 
& $\boldsymbol{\tau=3}$ 
& $\boldsymbol{\tau=5}$ \\
\midrule
\multirow{4}{*}{NFs}
& PI           & $0.338 \pm 0.005$ & $0.417 \pm 0.006$ & $0.459 \pm 0.014$ \\
& RA           & $0.329 \pm 0.004$ & $0.405 \pm 0.007$ & $0.453 \pm 0.015$ \\
& IPTW        & $0.319 \pm 0.006$ & $0.410 \pm 0.004$ & $0.457 \pm 0.012$ \\
& \OrthoGen    & $\mathbf{0.307 \pm 0.007}$ & $\mathbf{0.387 \pm 0.007}$ & $\mathbf{0.442 \pm 0.013}$ \\
\midrule
\multirow{4}{*}{DMs}
& PI           & $\mathbf{0.314 \pm 0.008}$ & $0.428 \pm 0.005$ & $0.500 \pm 0.007$ \\
& RA           & $0.330 \pm 0.007$ & $0.446 \pm 0.005$ & $0.506 \pm 0.005$ \\
& IPTW        & $0.347 \pm 0.020$ & $\mathbf{0.426 \pm 0.019}$ & $\mathbf{0.472 \pm 0.011}$ \\
& \OrthoGen    & $0.330 \pm 0.017$ & $0.432 \pm 0.018$ & $0.475 \pm 0.020$ \\
\bottomrule
\end{tabular}%
}
\par\vspace{1pt}
\begin{minipage}{\linewidth}
\raggedright\tiny
Reported are factual matching CRPS values across treatment horizons $\tau$.\par
Values are mean $\pm$ standard deviation over 5 replicates; lower is better.
\vspace{-0.9em}
\end{minipage}
\end{wraptable}

\textbf{Results.} Overall, our results demonstrate the theoretical benefits of \OrthoGen across different datasets and generative backbones. \circledgreen{1} On the \underline{\textit{fully synthetic}} benchmark (Table~\ref{tab:wasserstein_tau_synthetic}), \OrthoGen performs particularly strongly with DMs and remains highly competitive with NFs. Importantly, for $\tau>1$, \OrthoGen consistently outperforms the PI learner, with its advantage becoming increasingly pronounced at longer treatment horizons. This pattern is consistent with the robustness benefits predicted by our theoretical analysis: as the treatment horizon grows and nuisance estimation becomes more challenging, the doubly robust and Neyman-orthogonal adjustment of \OrthoGen yields increasingly substantial gains over the alternative learners.
\circledgreen{2} On the \underline{\textit{semi-synthetic}} benchmark (Table~\ref{tab:wasserstein_tau_semisynthetic}), \OrthoGen consistently achieves the lowest Wasserstein cost across \textbf{all} treatment horizons and for \textbf{both} NFs and DMs, and the gains again become larger for longer horizons, which confirms the theoretical properties of \OrthoGen due to Neyman-orthogonality. \circledgreen{3} Finally, on the \underline{\textit{real-world}} benchmark (Table~\ref{tab:crps_tau_real}), \OrthoGen achieves the lowest factual CRPS across all horizons with NFs. Again, \OrthoGen is effective across different backbones. $\Rightarrow$ \textit{Taken together, these results demonstrate the benefits of the doubly robust and Neyman-orthogonal construction of \OrthoGen, which yields substantial improvements over the alternative learners, particularly at longer treatment horizons. }

\textbf{Conclusion.} We introduced \OrthoGen, the first doubly robust and Neyman-orthogonal generative learner for CDPO estimation under time-varying treatments, built on our generative recursive g-computation strategy. We showed that \OrthoGen achieves rate double robustness and quasi-oracle efficiency, and performs effectively across synthetic, semi-synthetic, and real-world benchmarks with different generative backbones.

\section*{Acknowledgments}
We gratefully acknowledge Novartis for sponsoring Tomàs Garriga’s industrial PhD. We also thank the Government of Catalonia’s Industrial PhDs Plan for funding part of this research. We acknowledge Horizon Europe Programme under the AI4DEBUNK Project (\url{https://www.ai4debunk.eu}), grant agreement num. 101135757.  This work has also been supported by the German Federal Ministry of Education and Research (Grant: 01IS24082).



\bibliography{iclr2026_conference}
\bibliographystyle{iclr2026_conference}

\appendix
\newpage

\startcontents[appendices]

\section*{Appendix contents}
\printcontents[appendices]{}{1}{
    \setcounter{tocdepth}{2}
}

\vspace{1em}
\newpage

\section{Extended Related Work}
\label{app:extended-related-work}
\textbf{Learning CDPOs in static settings.}
Most causal machine-learning methods for static treatments have traditionally targeted conditional average potential outcomes (CAPOs), or contrasts such as conditional average treatment effects (CATEs), using methods such as S-, T-, and X-learners \citep{kunzel2019metalearners} or representation-learning approaches \citep{shalit2017estimating}. More recent work moves beyond these mean-based targets to learn potential-outcome distributions and thereby characterize treatment effects beyond the mean. This literature considers both marginal potential-outcome distributions and covariate-conditional potential-outcome distributions (CDPOs). At the marginal level, \citet{kennedy2023semiparametric} develop semiparametric one-step, doubly robust estimators for projections of counterfactual densities. Building on this framework, \citet{melnychuk2023normalizing} introduce Interventional Normalizing Flows, which instantiate the corresponding bias correction using normalizing flows. For CDPOs, GANITE uses generative adversarial networks with regression adjustment \citep{yoon2018ganite}, while NOFLITE \citep{vanderschueren2023noflite} and PO-Flow \citep{wu2026flowbasedgenerativemodelingpotential} are flow-based plug-in learners. DiffPO instead combines conditional diffusion models with inverse probability of treatment weighting (IPTW) and is Neyman-orthogonal only under a realizability condition \citep{ma2024diffpo}. A neighboring line uses plug-in learning to model the entire causal data-generating process with different generative models \citep{NIPS2017_94b5bde6,NEURIPS2020_0987b8b3,parafita2022estimand,sanchez2022diffusion}. Relatedly, \citet{garriga2026cepae} develop an autoencoder-based approach that uses structural causal models for counterfactual inference in time series under interventions on observed events. Most recently, \citet{melnychuk2025gdr} and \citet{luedtke2025doublegen} introduced Neyman-orthogonal and doubly robust learners for CDPOs in the static setting. 

\textbf{Learning CAPOs in time-varying settings.} Classical approaches for CAPO estimation under time-varying treatments include marginal structural models estimated through IPTW \citep{robins2000marginal} and the parametric g-formula, or g-computation \citep{taubman2009intervening}. More recent methods adapt these strategies to complex longitudinal data with deep learning. \citet{RMSN} combines recurrent models with propensity weighting. \citet{gnet} implement simulation-based g-computation with recurrent networks, while G-Transformer \citep{pmlr-v252-xiong24a} uses Transformer-based sequence models. IGC-Net instead performs iterative regression-based g-computation with a Transformer architecture \citep{hess2026igcnet}. Counterfactual Recurrent Networks \citep{Bica2020Estimating} and Causal Transformer \citep{melnychuk2022causal} learn balanced representations of patient histories. Moving beyond architecture-specific estimators, \citet{frauen2025modelagnostic} introduce model-agnostic meta-learners based on history adjustment, regression adjustment, propensity weighting, doubly robust correction, and inverse-variance weighting, while \citet{hess2026overlapweighted} extend the previous work by introducing an orthogonal overlap-weighted estimator that addresses the diminishing overlap associated with long treatment sequences.

\section{Generative Model Instantiations}
\label{app:generative-model-instantiations}

We consider two conditional generative-model families as instantiations of
the framework: normalizing flows and diffusion models. Following the
score-based formulation of \citet{melnychuk2025gdr}, we describe both the
probabilistic model and the population distributional criterion induced by
its score. Throughout this section, $C$ denotes the conditioning variables:
$C=\bar H_{t+\delta}$ for a continuation model at relative step $\delta$, and $C=\bar H_t$ for the
target generator. We write
$P_{\bar a}(\cdot\mid\bar h_t)
:=P(\Ybarat\in\cdot\mid\bar H_t=\bar h_t)$ for the target CDPO.

\subsection{Conditional normalizing flows}

\textbf{Probabilistic model.}
A conditional normalizing flow maps a base variable $U\sim p_0$ through an
invertible differentiable transformation $f_\theta(\cdot;c)$
\citep{papamakarios2021normalizing}. For
$Y=f_\theta(U;c)$, the change-of-variables formula yields
\begin{equation}
p_\theta(y\mid c)
=
p_0\!\left(f_\theta^{-1}(y;c)\right)
\left|
\det\!\left(
\frac{\partial f_\theta^{-1}(y;c)}{\partial y}
\right)
\right|.
\label{eq:nf-conditional-density}
\end{equation}
Thus, both sampling and density evaluation are available without integrating
over $U$. In the notation of Eq.~\ref{eq:generative-score}, the auxiliary
space is degenerate,
$g_{\bar a}(y,z\mid c)=p_\theta(y\mid c)$, and
$\ell_{g_{\bar a}}(y,c)=\log p_\theta(y\mid c)$.

\textbf{Distributional interpretation.}
Let $\mathcal P_{\mathrm{NF}}$ be the chosen class of conditional flow
densities. If the target CDPO has density $p_{\bar a}$ and finite conditional
entropy
$\mathbb{H}(P_{\bar a})
=-\mathbb E[\log p_{\bar a}(\Ybarat\mid\bar H_t)]$, then the population
flow criterion satisfies
\begin{align}
\mathcal L_{\bar a}^{\mathrm{NF}}(p_\theta)
&:=
\mathbb E\!\left[
\log p_\theta(\Ybarat\mid\bar H_t)
\right]
\nonumber\\
&=
-\mathbb{H}(P_{\bar a})
-
\mathbb E_{\bar H_t}\!\left[
\mathrm{KL}\!\left(
P_{\bar a}(\cdot\mid\bar H_t)
\,\Vert\,
p_\theta(\cdot\mid\bar H_t)
\right)
\right].
\label{eq:nf-risk-kl-decomposition}
\end{align}
The entropy term does not depend on $p_\theta$. Consequently,
\begin{align}
\arg\max_{p_\theta\in\mathcal P_{\mathrm{NF}}}
\mathcal L_{\bar a}^{\mathrm{NF}}(p_\theta)
=
\arg\min_{p_\theta\in\mathcal P_{\mathrm{NF}}}
\mathbb E_{\bar H_t}\!\left[
\mathrm{KL}\!\left(
P_{\bar a}(\cdot\mid\bar H_t)
\,\Vert\,
p_\theta(\cdot\mid\bar H_t)
\right)
\right].
\label{eq:nf-forward-kl-projection}
\end{align}
Conditional maximum likelihood therefore returns the forward-KL projection
of the CDPO onto the flow class.

\subsection{Conditional diffusion models}

\textbf{Probabilistic model.}
Diffusion models specify a fixed forward noising process and learn its
conditional reversal \citep{yang2023diffusion}. Let $j=1,\ldots,S$ index
diffusion steps, separately from longitudinal time, and set $Z_0=Y$. For a
schedule $\alpha_j\in(0,1)$ and
$\bar\alpha_j=\prod_{r=1}^j\alpha_r$, the forward chain is
\begin{align}
q(z_{1:S}\mid y)
&=
\prod_{j=1}^{S}q_j(z_j\mid z_{j-1}),
\nonumber\\
q_j(z_j\mid z_{j-1})
&=
\mathcal N\!\left(
z_j;
\sqrt{\alpha_j}z_{j-1},
(1-\alpha_j)I
\right).
\label{eq:dm-forward-chain}
\end{align}
Its $j$th marginal can be sampled directly as
\begin{equation}
Z_j
=
\sqrt{\bar\alpha_j}\,Y
+
\sqrt{1-\bar\alpha_j}\,\varepsilon,
\qquad
\varepsilon\sim\mathcal N(0,I).
\label{eq:dm-forward-noising}
\end{equation}
For a schedule ending close to pure noise, $Z_S$ is approximately standard
Gaussian. The learned reverse model has joint density
\begin{equation}
p_\theta(y,z_{1:S}\mid c)
=
p_S(z_S)
\prod_{j=1}^{S}
p_\theta(z_{j-1}\mid z_j,c),
\qquad z_0=y,
\label{eq:dm-reverse-joint}
\end{equation}
where $p_S$ is usually standard Gaussian. Samples are generated from $p_S$
and passed through the reverse transitions from $S$ down to $1$.

\textbf{Variational score and distributional interpretation.}
The marginal likelihood of $Y$ requires integration over the noising path.
To obtain a tractable score, set
\begin{equation}
g_{\bar a,\theta}(y,z_{1:S}\mid c)
=
\frac{p_\theta(y,z_{1:S}\mid c)}{q(z_{1:S}\mid y)},
\qquad
\eps_z(\dd z_{1:S}\mid y,c)
=q(z_{1:S}\mid y)\dd z_{1:S}.
\label{eq:dm-generative-representation}
\end{equation}
Equation~\ref{eq:generative-score} then becomes the conditional evidence
lower bound (ELBO):
\begin{align}
\ell_{g_{\bar a,\theta}}(y,c)
&=
\mathbb E_{Z_{1:S}\sim q(\cdot\mid y)}\!\left[
\log
\frac{p_\theta(y,Z_{1:S}\mid c)}
{q(Z_{1:S}\mid y)}
\right]
\nonumber\\
&=
\log p_\theta(y\mid c)
-
\mathrm{KL}\!\left(
q(Z_{1:S}\mid y)
\,\Vert\,
p_\theta(Z_{1:S}\mid y,c)
\right)
\leq
\log p_\theta(y\mid c).
\label{eq:dm-elbo}
\end{align}
Define the average posterior (or inference gap) by
\begin{align}
\mathrm{IG}_{\bar a}(\theta)
:=
\mathbb E_{\bar H_t}
\mathbb E_{Y\sim P_{\bar a}(\cdot\mid\bar H_t)}\!\left[
\mathrm{KL}\!\left(
q(Z_{1:S}\mid Y)
\,\Vert\,
p_\theta(Z_{1:S}\mid Y,\bar H_t)
\right)
\right].
\label{eq:dm-inference-gap}
\end{align}
Averaging Eq.~\ref{eq:dm-elbo} under the target CDPO gives
\begin{align}
\mathcal L_{\bar a}^{\mathrm{DM}}(\theta)
&:=
\mathbb E\!\left[
\ell_{g_{\bar a,\theta}}(\Ybarat,\bar H_t)
\right]
\nonumber\\
&=
-\mathbb{H}(P_{\bar a})
-
\mathbb E_{\bar H_t}\!\left[
\mathrm{KL}\!\left(
P_{\bar a}(\cdot\mid\bar H_t)
\,\Vert\,
p_\theta(\cdot\mid\bar H_t)
\right)
\right]
-\mathrm{IG}_{\bar a}(\theta).
\label{eq:dm-risk-kl-decomposition}
\end{align}
Hence the population diffusion target obeys
\begin{align}
\arg\max_{\theta\in\Theta_{\mathrm{DM}}}
\mathcal L_{\bar a}^{\mathrm{DM}}(\theta)
=
\arg\min_{\theta\in\Theta_{\mathrm{DM}}}
\Bigg\{
&\mathbb E_{\bar H_t}\!\left[
\mathrm{KL}\!\left(
P_{\bar a}(\cdot\mid\bar H_t)
\,\Vert\,
p_\theta(\cdot\mid\bar H_t)
\right)
\right]
+\mathrm{IG}_{\bar a}(\theta)
\Bigg\}.
\label{eq:dm-distance-projection}
\end{align}
Unlike the flow criterion, the diffusion ELBO contains the additional
posterior gap. It vanishes when the fixed forward law agrees with the model
posterior; more generally, it records the price of using $q$ as a
variational approximation.

For Gaussian reverse transitions with fixed variances, the terms involving
their means can be written as weighted noise-prediction errors. A common
practical objective samples $J$ uniformly from $\{1,\ldots,S\}$ and uses
\begin{equation}
\mathcal J_{\mathrm{DM}}(\theta)
=
\mathbb E_{Y,C,J,\varepsilon}
\left[
\left\|
\varepsilon-
\varepsilon_\theta\!\left(
\sqrt{\bar\alpha_J}Y+
\sqrt{1-\bar\alpha_J}\varepsilon,
J,C
\right)
\right\|_2^2
\right].
\label{eq:dm-denoising-loss}
\end{equation}
The exactly weighted objective, together with its endpoint terms, is the
negative ELBO in Eq.~\ref{eq:dm-elbo}; the uniform-weight version in
Eq.~\ref{eq:dm-denoising-loss} is the usual simplified surrogate. In either
case, the per-observation negative denoising loss can be used as the
model-specific score in Eq.~\ref{eq:generative-score}, with the sampled
diffusion step and Gaussian noise serving as the auxiliary variables.

For generative recursive g-computation, these model-specific criteria are applied at each relative step $\delta$ with
$C=\bar H_{t+\delta}$ and with either the observed terminal outcome or the generated
pseudo-outcome as the response. For the target generator, $C=\bar H_t$;
the RA, IPTW, and \OrthoGen risks replace the inaccessible expectation under
$P_{\bar a}$ by the observational representations developed in the main
text.

\section{Additional Learner Details and Risk Representations}
\label{app:additional-learner-details}

\subsection{IPTW learner}
\label{app:iptw-learner}

\textbf{IPTW learner.}
An alternative representation of the oracle risk is obtained through inverse probability of treatment weighting (IPTW). We define the time-specific propensity score
$\pi_{t+\delta}(a_{t+\delta}\mid\bar H_{t+\delta})=P(A_{t+\delta}=a_{t+\delta}\mid\bar H_{t+\delta})$, for $\delta=0,\ldots,\tau-1$, and the one-step and cumulative inverse-propensity ratios, respectively, as:
\begin{equation}
r_{t+\delta}^{\bar a}
=
\frac{\ind\{A_{t+\delta}=a_{t+\delta}\}}
{\pi_{t+\delta}(a_{t+\delta}\mid\bar H_{t+\delta})},
\qquad
R_{t:t+\delta}^{\bar a}
=
\prod_{\delta'=0}^{\delta}r_{t+\delta'}^{\bar a},
\qquad \delta=0,\ldots,\tau-1,
\qquad
R_{t:t-1}^{\bar a}=1.
\label{eq:iptw-cumulative-treatments}
\end{equation}
Under the identification assumptions of Sec.~\ref{sec:problem-setting}, the oracle target risk in Eq.~\ref{eq:oracle-risk} can be written as
$
\mathcal L_{\bar a}(g_{\bar a}) =
\mathbb E
\left[
R_{t:t+\tau-1}^{\bar a}
\ell_{g_{\bar a}}(Y_{t+\tau},\bar H_t)
\right]$.
Let $\bpi$, $\br$, and $\bR$ denote the corresponding estimated quantities. The empirical IPTW risk is
\begin{align}
\widehat{\mathcal L}_{\mathrm{IPTW}}^{\bar a}(g_{\bar a})
=
\Pb\left[
\bR_{t:t+\tau-1}^{\bar a}
\ell_{g_{\bar a}}(Y_{t+\tau},\bar H_t)
\right].
\label{eq:ipw-risk}
\end{align}
Unlike the RA learner, the IPTW learner does not rely on the continuation nuisances and instead requires consistent estimation of the treatment mechanism.

\subsection{Oracle-risk representation of the RA learner}
\label{app:ra-risk-representation}

\begin{lemma}[Population equivalence of the RA risk]
\label{lem:ra-risk-equivalence}
Suppose that the standard causal assumptions stated in
Sec.~\ref{sec:problem-setting} hold and that the relevant quantities are
integrable. Let the continuation densities
$\{\xi_{t+\delta}^{\bar a}\}_{\delta=0}^{\tau-1}$ be defined as in
Sec.~\ref{sec:generative-recursive-gcomp}, and let $m_{t+\delta}^{g_{\bar a},\bar a}$ be the corresponding
population score responses defined in Eq.~\ref{eq:ra-score-response}.
Then, for every $g_{\bar a}\in\cG$,
\begin{align}
\mathcal L_{\bar a}(g_{\bar a})
&=
\mathbb E\left[
m_t^{g_{\bar a},\bar a}(\bar H_t)
\right].
\label{eq:ra-population-risk-equivalence}
\end{align}
Hence, the population risk underlying the RA learner is
identical to the oracle target generative risk in
Eq.~\ref{eq:oracle-risk}.
\end{lemma}

\begin{proof}
Starting from Eq.~\ref{eq:oracle-risk},
\begin{align}
\mathcal L_{\bar a}(g_{\bar a})
&=
\mathbb E\!\left[
\ell_{g_{\bar a}}(\Ybarat,\bar H_t)
\right] \\
&=
\mathbb E\!\left[
\mathbb E\!\left[
\ell_{g_{\bar a}}(\Ybarat,\bar H_t)
\mid \bar H_t
\right]
\right] \\
&=
\mathbb E\!\left[
\int_{\mathcal Y}
\ell_{g_{\bar a}}(y,\bar H_t)
P(\Ybarat=y\mid\bar H_t)\,dy
\right] \\
&=
\mathbb E\!\left[
\int_{\mathcal Y}
\ell_{g_{\bar a}}(y,\bar H_t)
\xi_t^{\bar a}(y\mid\bar H_t)\,dy
\right] \\
&=
\mathbb E\!\left[
m_t^{g_{\bar a},\bar a}(\bar H_t)
\right].
\end{align}
The second equality is the tower property, the third expresses the
conditional expectation through the conditional law of \(\Ybarat\), the
fourth invokes Proposition~\ref{prop:recursive-gcomp-identification}, and
the last is Eq.~\ref{eq:ra-score-response}.
\end{proof}

\subsection{Oracle-risk representation of the IPTW learner}
\label{app:iptw-risk-representation}

\begin{lemma}[Population equivalence of the IPTW risk]
\label{lem:iptw-risk-equivalence}
Suppose that the standard causal assumptions stated in
Sec.~\ref{sec:problem-setting} hold and that the relevant quantities are
integrable. Let the cumulative inverse-propensity ratios
$R_{t:t+\delta}^{\bar a}$ be defined as in Sec.~\ref{sec:orthogen}.
Then, for every $g_{\bar a}\in\cG$,
\begin{align}
\mathcal L_{\bar a}(g_{\bar a})
&=
\mathbb E
\left[
R_{t:t+\tau-1}^{\bar a}
\ell_{g_{\bar a}}(Y_{t+\tau},\bar H_t)
\right].
\label{eq:iptw-population-risk-equivalence}
\end{align}
Hence, the population risk underlying the IPTW learner is identical to
the oracle target generative risk in Eq.~\ref{eq:oracle-risk}.
\end{lemma}

\begin{proof}
The stated integrability condition ensures that all conditional
expectations below are well defined. For every relevant integrable
random variable \(Z\), positivity and the definition of the propensity
score give, almost surely,
\begin{align}
\mathbb E\!\left[
r_{t+\delta}^{\bar a}Z
\mid \bar H_{t+\delta}
\right]
&=
\frac{1}{\pi_{t+\delta}(a_{t+\delta}\mid\bar H_{t+\delta})}
\mathbb E\!\left[
\ind\{A_{t+\delta}=a_{t+\delta}\}Z
\mid \bar H_{t+\delta}
\right] \\
&=
\frac{P(A_{t+\delta}=a_{t+\delta}\mid\bar H_{t+\delta})}
{\pi_{t+\delta}(a_{t+\delta}\mid\bar H_{t+\delta})}
\mathbb E\!\left[
Z\mid\bar H_{t+\delta},A_{t+\delta}=a_{t+\delta}
\right] \\
&=
\mathbb E\!\left[
Z\mid\bar H_{t+\delta},A_{t+\delta}=a_{t+\delta}
\right].
\label{eq:proof-basic-conditional-weighting}
\end{align}
Thus Eq.~\ref{eq:proof-basic-conditional-weighting} is an
inverse-propensity weighting identity; positivity justifies every
division by \(\pi_{t+\delta}(a_{t+\delta}\mid\bar H_{t+\delta})\).

Define
\begin{align}
Q_{t+\tau}(\bar H_{t+\tau})
&:=
\ell_{g_{\bar a}}(Y_{t+\tau},\bar H_t), \\
Q_{t+\delta}(\bar H_{t+\delta})
&:=
\mathbb E\!\left[
Q_{t+\delta+1}(\bar H_{t+\delta+1})
\mid \bar H_{t+\delta},A_{t+\delta}=a_{t+\delta}
\right],
\qquad \delta=\tau-1,\ldots,0.
\label{eq:proof-iptw-Q-recursion}
\end{align}
Because \(\bar H_{t+\delta}\) is the pre-interventional history, it contains
all treatments through time \(t+\delta-1\); hence
\(R_{t:t+\delta-1}^{\bar a}\) is
\(\bar H_{t+\delta}\)-measurable. Consequently, for each
\(\delta=0,\ldots,\tau-1\),
\begin{align}
&\mathbb E\!\left[
R_{t:t+\delta}^{\bar a}Q_{t+\delta+1}(\bar H_{t+\delta+1})
\right] \\
&\quad=
\mathbb E\!\left[
R_{t:t+\delta-1}^{\bar a}
\mathbb E\!\left[
r_{t+\delta}^{\bar a}Q_{t+\delta+1}(\bar H_{t+\delta+1})
\mid\bar H_{t+\delta}
\right]
\right] \\
&\quad=
\mathbb E\!\left[
R_{t:t+\delta-1}^{\bar a}
\mathbb E\!\left[
Q_{t+\delta+1}(\bar H_{t+\delta+1})
\mid\bar H_{t+\delta},A_{t+\delta}=a_{t+\delta}
\right]
\right] \\
&\quad=
\mathbb E\!\left[
R_{t:t+\delta-1}^{\bar a}Q_{t+\delta}(\bar H_{t+\delta})
\right].
\label{eq:proof-peel-iptw-factor}
\end{align}
The first equality uses the tower property and the measurability of the
previous-stage cumulative weight, the second uses
Eq.~\ref{eq:proof-basic-conditional-weighting}, and the third uses the
definition of \(Q_{t+\delta}\). Peeling
Eq.~\ref{eq:proof-peel-iptw-factor} backward from
\(\delta=\tau-1\) through \(\delta=0\), and using the empty-product
convention \(R_{t:t-1}^{\bar a}=1\), yields
\begin{align}
\mathbb E\!\left[
R_{t:t+\tau-1}^{\bar a}\ell_{g_{\bar a}}(Y_{t+\tau},\bar H_t)
\right]
&=
\mathbb E\!\left[
Q_t(\bar H_t)
\right].
\label{eq:proof-iptw-to-Qt}
\end{align}

It remains to give the causal interpretation of \(Q_t\). We show by
backward induction that, for \(\delta=0,\ldots,\tau-1\),
\begin{align}
Q_{t+\delta}(\bar H_{t+\delta})
&=
\mathbb E\!\left[
\ell_{g_{\bar a}}\!\left(
Y_{t+\tau}[a_{t+\delta:t+\tau-1}],\bar H_t
\right)
\mid\bar H_{t+\delta}
\right].
\label{eq:proof-Qt-causal-interpretation}
\end{align}
At \(\delta=\tau-1\), the recursion in
Eq.~\ref{eq:proof-iptw-Q-recursion} gives
\begin{align}
Q_{t+\tau-1}(\bar H_{t+\tau-1})
&=
\mathbb E\!\left[
\ell_{g_{\bar a}}(Y_{t+\tau},\bar H_t)
\mid\bar H_{t+\tau-1},A_{t+\tau-1}=a_{t+\tau-1}
\right] \\
&=
\mathbb E\!\left[
\ell_{g_{\bar a}}(Y_{t+\tau}[a_{t+\tau-1}],\bar H_t)
\mid\bar H_{t+\tau-1},A_{t+\tau-1}=a_{t+\tau-1}
\right] \\
&=
\mathbb E\!\left[
\ell_{g_{\bar a}}(Y_{t+\tau}[a_{t+\tau-1}],\bar H_t)
\mid\bar H_{t+\tau-1}
\right].
\label{eq:proof-iptw-causal-base}
\end{align}
The second equality is terminal recursive consistency on
\(A_{t+\tau-1}=a_{t+\tau-1}\), and the third is sequential
exchangeability at time \(t+\tau-1\).

Now suppose Eq.~\ref{eq:proof-Qt-causal-interpretation} holds at
\(t+\delta+1\), for some \(\delta<\tau-1\). Then
\begin{align}
Q_{t+\delta}(\bar H_{t+\delta})
&=
\mathbb E\!\left[
Q_{t+\delta+1}(\bar H_{t+\delta+1})
\mid\bar H_{t+\delta},A_{t+\delta}=a_{t+\delta}
\right] \\
&=
\mathbb E\!\left[
\mathbb E\!\left[
\ell_{g_{\bar a}}\!\left(
Y_{t+\tau}[a_{t+\delta+1:t+\tau-1}],\bar H_t
\right)
\mid\bar H_{t+\delta+1}
\right]
\mathrel{\Big|}
\bar H_{t+\delta},A_{t+\delta}=a_{t+\delta}
\right] \\
&=
\mathbb E\!\left[
\mathbb E\!\left[
\ell_{g_{\bar a}}\!\left(
Y_{t+\tau}[a_{t+\delta:t+\tau-1}],\bar H_t
\right)
\mid\bar H_{t+\delta+1}
\right]
\mathrel{\Big|}
\bar H_{t+\delta},A_{t+\delta}=a_{t+\delta}
\right] \\
&=
\mathbb E\!\left[
\ell_{g_{\bar a}}\!\left(
Y_{t+\tau}[a_{t+\delta:t+\tau-1}],\bar H_t
\right)
\mid\bar H_{t+\delta},A_{t+\delta}=a_{t+\delta}
\right] \\
&=
\mathbb E\!\left[
\ell_{g_{\bar a}}\!\left(
Y_{t+\tau}[a_{t+\delta:t+\tau-1}],\bar H_t
\right)
\mid\bar H_{t+\delta}
\right].
\label{eq:proof-iptw-causal-step}
\end{align}
The first equality is the definition of \(Q_{t+\delta}\), and the second is the
induction hypothesis. For the third equality, recursive
consistency/composition gives
\(Y_{t+\tau}[a_{t+\delta:t+\tau-1}]
=Y_{t+\tau}[a_{t+\delta+1:t+\tau-1}]\) on
\(A_{t+\delta}=a_{t+\delta}\).
The fourth equality is the tower property, since \(\bar H_{t+\delta+1}\)
contains \((\bar H_{t+\delta},A_{t+\delta})\), and the fifth is sequential
exchangeability at time \(t+\delta\). In particular, \(\bar H_t\) is
contained in every \(\bar H_{t+\delta}\), so it is measurable in all these
conditioning steps and
sequential exchangeability for the potential outcome also applies to
the displayed score transformation.

Taking \(\delta=0\) in Eq.~\ref{eq:proof-Qt-causal-interpretation} and using
the tower property therefore gives
\begin{align}
\mathbb E\!\left[Q_t(\bar H_t)\right]
&=
\mathbb E\!\left[
\ell_{g_{\bar a}}\!\left(
Y_{t+\tau}[a_{t:t+\tau-1}],\bar H_t
\right)
\right] \\
&=
\mathbb E\!\left[
\ell_{g_{\bar a}}(\Ybarat,\bar H_t)
\right]
=
\mathcal L_{\bar a}(g_{\bar a}),
\end{align}
where \(a_{t:t+\tau-1}=\barat\). Combining this equality
with Eq.~\ref{eq:proof-iptw-to-Qt} proves the claim.
\end{proof}

\section{Generative Recursive G-Computation}
\label{app:generative-recursive-gcomp-details}

\subsection{Pseudocode for generative recursive g-computation}
\label{app:generative-recursive-gcomp-pseudocode}

Algorithm~\ref{alg:generative-recursive-gcomp} presents the pseudocode for generative recursive g-computation.

\begin{algorithm}[H]
\caption{Generative recursive g-computation}
\label{alg:generative-recursive-gcomp}
\begin{algorithmic}[1]

\Require Longitudinal training sample $\mathcal D$, prediction origin $t$,
horizon $\tau$, fixed intervention $\barat$,
parameter spaces $\{\Phi_{t+\delta}\}_{\delta=0}^{\tau-1}$ for the
nuisance generators $\xi_{\phi,t+\delta}^{\bar a}$,
model-specific generative score $\ell$, and Monte Carlo size $M$.

\Ensure Fitted continuation laws
$\{\bxi_{t+\delta}^{\bar a}\}_{\delta=0}^{\tau-1}$.

\State $\mathcal I_{t+\tau-1}\gets\{i:A_{t+\tau-1}^{(i)}=a_{t+\tau-1}\}$

\State Fit the final-stage nuisance generator:
\Statex \hspace{1em}$\displaystyle
\widehat\phi_{t+\tau-1}
\in
\arg\max_{\phi\in\Phi_{t+\tau-1}}
\frac{1}{|\mathcal I_{t+\tau-1}|}
\sum_{i\in\mathcal I_{t+\tau-1}}
\ell_{\xi_{\phi,t+\tau-1}^{\bar a}}
\left(
Y_{t+\tau}^{(i)},
\bar H_{t+\tau-1}^{(i)}
\right).
$

\State
$\bxi_{t+\tau-1}^{\bar a}
\gets
\xi_{\widehat\phi_{t+\tau-1},t+\tau-1}^{\bar a}$

\For{$\delta=\tau-2,\tau-3,\ldots,0$}

    \State
  $\mathcal I_{t+\delta}\gets\{i:A_{t+\delta}^{(i)}=a_{t+\delta}\}$

  \State Hold $\bxi_{t+\delta+1}^{\bar a}$ fixed as the teacher distribution.

  \For{$i\in\mathcal I_{t+\delta}$}
        \For{$m=1,\ldots,M$}
            \State Draw
      $\widetilde Y_{i,t+\delta+1}^{(m)}
            \sim
      \bxi_{t+\delta+1}^{\bar a}
      (\cdot\mid\bar H_{t+\delta+1}^{(i)})$
        \EndFor
    \EndFor

  \State Fit the time-$t+\delta$ nuisance generator to the pseudo-outcomes:
    \Statex \hspace{1em}$\displaystyle
  \widehat\phi_{t+\delta}
    \in
  \arg\max_{\phi\in\Phi_{t+\delta}}
  \frac{1}{|\mathcal I_{t+\delta}|M}
  \sum_{i\in\mathcal I_{t+\delta}}
    \sum_{m=1}^{M}
  \ell_{\xi_{\phi,t+\delta}^{\bar a}}
    \left(
  \widetilde Y_{i,t+\delta+1}^{(m)},
  \bar H_{t+\delta}^{(i)}
    \right).
    $

    \State
  $\bxi_{t+\delta}^{\bar a}
    \gets
  \xi_{\widehat\phi_{t+\delta},t+\delta}^{\bar a}$

\EndFor

\State \Return $\{\bxi_{t+\delta}^{\bar a}\}_{\delta=0}^{\tau-1}$

\end{algorithmic}
\end{algorithm}

\subsection{Proof of Proposition \ref{prop:recursive-gcomp-identification}}
\label{app:generative-recursive-gcomp}

\begin{proof}
We first verify the desired representation at the terminal stage
$t+\tau-1$ and then use the recursion in Eq.~\ref{eq:xi-recursion} to
move successively from stage $t+\delta+1$ to stage $t+\delta$, for
$\delta=\tau-2,\ldots,0$. Specifically, we establish that
\begin{align}
\xi_{t+\delta}^{\bar a}(y\mid\bar h_{t+\delta})
=P\!\left(
Y_{t+\tau}[a_{t+\delta:t+\tau-1}]=y
\mid
\bar H_{t+\delta}=\bar h_{t+\delta}
\right)
\label{eq:proof-stagewise-continuation}
\end{align}
for every $\delta=0,\ldots,\tau-1$ and every relevant
$\bar h_{t+\delta}$. Throughout,
positivity guarantees that the conditional distributions given
$A_{t+\delta}=a_{t+\delta}$ appearing below are identified on the relevant
support.

At the terminal stage, Eq.~\ref{eq:terminal-stage-xi}, recursive
consistency, and sequential exchangeability give
\begin{align}
\xi_{t+\tau-1}^{\bar a}(y\mid\bar h_{t+\tau-1})
&=
P\!\left(
Y_{t+\tau}=y
\mid
\bar H_{t+\tau-1}=\bar h_{t+\tau-1},
A_{t+\tau-1}=a_{t+\tau-1}
\right) \\
&=
P\!\left(
Y_{t+\tau}[a_{t+\tau-1}]=y
\mid
\bar H_{t+\tau-1}=\bar h_{t+\tau-1},
A_{t+\tau-1}=a_{t+\tau-1}
\right) \\
&=
P\!\left(
Y_{t+\tau}[a_{t+\tau-1}]=y
\mid
\bar H_{t+\tau-1}=\bar h_{t+\tau-1}
\right).
\end{align}
The first equality is the terminal condition. The next equality uses the
terminal case of recursive consistency: on the event
$A_{t+\tau-1}=a_{t+\tau-1}$,
$Y_{t+\tau}[a_{t+\tau-1}]=Y_{t+\tau}$. The final equality removes
$A_{t+\tau-1}=a_{t+\tau-1}$ by sequential exchangeability.

Now consider any $\delta<\tau-1$ for which
Eq.~\ref{eq:proof-stagewise-continuation} has already been established at
stage $t+\delta+1$. For a relevant history $\bar h_{t+\delta}$, the
recursive definition and the representation at stage $t+\delta+1$ give
\begin{align}
\xi_{t+\delta}^{\bar a}(y\mid\bar h_{t+\delta})
&=
\mathbb E\!\left[
\xi_{t+\delta+1}^{\bar a}(y\mid\bar H_{t+\delta+1})
\mathrel{\Big|}
\bar H_{t+\delta}=\bar h_{t+\delta},
A_{t+\delta}=a_{t+\delta}
\right] \\
&=
\mathbb E\!\left[
P\!\left(
Y_{t+\tau}[a_{t+\delta+1:t+\tau-1}]=y
\mid
\bar H_{t+\delta+1}
\right)
\mathrel{\Big|}
\bar H_{t+\delta}=\bar h_{t+\delta},
A_{t+\delta}=a_{t+\delta}
\right] \\
&=
\mathbb E\!\left[
P\!\left(
Y_{t+\tau}[a_{t+\delta:t+\tau-1}]=y
\mid
\bar H_{t+\delta+1}
\right)
\mathrel{\Big|}
\bar H_{t+\delta}=\bar h_{t+\delta},
A_{t+\delta}=a_{t+\delta}
\right] \\
&=
P\!\left(
Y_{t+\tau}[a_{t+\delta:t+\tau-1}]=y
\mid
\bar H_{t+\delta}=\bar h_{t+\delta},
A_{t+\delta}=a_{t+\delta}
\right) \\
&=
P\!\left(
Y_{t+\tau}[a_{t+\delta:t+\tau-1}]=y
\mid
\bar H_{t+\delta}=\bar h_{t+\delta}
\right).
\label{eq:proof-recursive-step}
\end{align}
The first equality is the recursion in Eq.~\ref{eq:xi-recursion}, and the
second uses the representation already obtained at stage $t+\delta+1$.
Under the conditioning event $A_{t+\delta}=a_{t+\delta}$, recursive
consistency gives
$Y_{t+\tau}[a_{t+\delta:t+\tau-1}]
=Y_{t+\tau}[a_{t+\delta+1:t+\tau-1}]$, which yields the third equality.
Because $\bar H_{t+\delta+1}$ contains both $\bar H_{t+\delta}$ and
$A_{t+\delta}$, the next
equality follows from the law of iterated expectations, applied to the
conditional density at $y$. The final equality again follows from
sequential exchangeability. Starting from the terminal-stage equality
and applying this recursive step for $\delta=\tau-2,\ldots,0$ establishes
Eq.~\ref{eq:proof-stagewise-continuation} at every stage. At $\delta=0$,
the remaining sequence is the full treatment sequence
$a_{t:t+\tau-1}=\barat$, and therefore
\begin{align}
\xi_t^{\bar a}(y\mid\bar h_t)
&=
P\!\left(
Y_{t+\tau}[a_{t:t+\tau-1}]=y
\mid
\bar H_t=\bar h_t
\right) \\
&=
P\!\left(
\Ybarat=y
\mid
\bar H_t=\bar h_t
\right),
\end{align}
which is the CDPO in Eq.~\ref{eq:target-law}.
\end{proof}

\section{Theoretical Properties of \OrthoGen}
\label{app:orthogen-theoretical-properties}

\subsection{Neyman-orthogonality}
\label{app:neyman-orthogonality}
Let \(\eta=(\pi,\xi)\) collect the true propensity-score and
continuation-density nuisances. Let
\(\widetilde\eta=(\widetilde\pi,\widetilde\xi)\) be any candidate pair of
conditional treatment laws and continuation densities for which the path
below is admissible.

For every \(g_{\bar a}\in\cG\), let
\(s\mapsto g_{\bar a,s}\in\cG\) be an admissible differentiable
target-model path with \(g_{\bar a,0}=g_{\bar a}^\star\) and tangent direction
\(g_{\bar a}-g_{\bar a}^\star\) at \(s=0\).

For any fixed \(0<\epsilon<1\), consider the nuisance path
\(\eta_u=(\pi_u,\xi_u)\), \(u\in[0,\epsilon)\), where, for every
\(\delta=0,\ldots,\tau-1\),
\begin{equation}
\begin{aligned}
\pi_{t+\delta,u}(\cdot\mid\bar h_{t+\delta})
&:=
(1-u)\pi_{t+\delta}(\cdot\mid\bar h_{t+\delta})
+
u\widetilde\pi_{t+\delta}(\cdot\mid\bar h_{t+\delta}),
\\
\xi_{t+\delta,u}^{\bar a}(y\mid\bar h_{t+\delta})
&:=
(1-u)\xi_{t+\delta}^{\bar a}(y\mid\bar h_{t+\delta})
+
u\widetilde\xi_{t+\delta}^{\bar a}(y\mid\bar h_{t+\delta}).
\end{aligned}
\label{eq:orth-mixture-path}
\end{equation}
Positivity gives
\begin{equation}
\pi_{t+\delta,u}(a_{t+\delta}\mid\bar H_{t+\delta})
\geq
(1-u)\pi_{t+\delta}(a_{t+\delta}\mid\bar H_{t+\delta})>0,
\qquad 0\leq u<\epsilon,
\end{equation}
so the path is locally positive at \(u=0\). Define
\begin{equation}
r_{u,t+\delta}^{\bar a}
:=
\frac{\ind\{A_{t+\delta}=a_{t+\delta}\}}
{\pi_{t+\delta,u}(a_{t+\delta}\mid\bar H_{t+\delta})},
\qquad
R_{u,t:t+\delta}^{\bar a}
:=
\prod_{\delta'=0}^{\delta}r_{u,t+\delta'}^{\bar a},
\qquad \delta=0,\ldots,\tau-1,
\qquad
R_{u,t:t-1}^{\bar a}:=1.
\label{eq:orth-path-weights}
\end{equation}

For \(g_{\bar a}\in\cG\), define
\begin{equation}
m_{t+\delta,u}^{g_{\bar a},\bar a}(\bar h_{t+\delta})
:=
\int_{\mathcal Y}
\ell_{g_{\bar a}}(y,\bar h_t)
\xi_{t+\delta,u}^{\bar a}(y\mid\bar h_{t+\delta})\dd y,
\qquad
\delta=0,\ldots,\tau-1,
\label{eq:orth-response}
\end{equation}
and
\begin{equation}
m_{t+\tau,u}^{g_{\bar a},\bar a}(\bar H_{t+\tau})
:=
\ell_{g_{\bar a}}(Y_{t+\tau},\bar H_t).
\label{eq:orth-terminal-response}
\end{equation}
Write \(m_{t+\delta}^{g_{\bar a},\bar a}
:=m_{t+\delta,0}^{g_{\bar a},\bar a}\) at the true nuisance.

Along this path, the population \OrthoGen objective is
\begin{equation}
\begin{aligned}
\mathcal L_{\mathrm{OrthoGen}}^{\bar a}(g_{\bar a};\eta_u)
:=
\mathbb E_0\Bigg[
&R_{u,t:t+\tau-1}^{\bar a}
\ell_{g_{\bar a}}(Y_{t+\tau},\bar H_t)
+ \sum_{\delta=0}^{\tau-1}
R_{u,t:t+\delta-1}^{\bar a}
\left(1-r_{u,t+\delta}^{\bar a}\right)
m_{t+\delta,u}^{g_{\bar a},\bar a}(\bar H_{t+\delta})
\Bigg].
\end{aligned}
\label{eq:orth-population-objective}
\end{equation}

Assume that the paths are locally admissible and that the score path has a
measurable right derivative and an integrable local envelope sufficient for
Fubini's theorem, conditional dominated convergence, and the differentiations
below. Endpoint derivatives are right directional derivatives, or
G\^ateaux derivatives for paths extending to a two-sided neighborhood.

Define the pathwise score derivative in the direction
\(g_{\bar a}-g_{\bar a}^\star\) by
\begin{equation}
\dot\ell_{g_{\bar a}}(y,\bar h_t)
:=
D_{g_{\bar a}}
\ell_{g_{\bar a}^\star}(y,\bar h_t)
\left[
g_{\bar a}-g_{\bar a}^\star
\right]
:=
\left.
\frac{d}{ds}
\ell_{g_{\bar a,s}}(y,\bar h_t)
\right|_{s=0}.
\label{eq:orth-score-derivative}
\end{equation}

For \(\delta=0,\ldots,\tau-1\), let
\begin{equation}
\begin{aligned}
\dot m_{t+\delta,u}^{g_{\bar a},\bar a}(\bar h_{t+\delta})
&:=
\left.
\frac{d}{ds}
m_{t+\delta,u}^{g_{\bar a,s},\bar a}(\bar h_{t+\delta})
\right|_{s=0}
\\
&=
\int_{\mathcal Y}
\dot\ell_{g_{\bar a}}(y,\bar h_t)
\xi_{t+\delta,u}^{\bar a}(y\mid\bar h_{t+\delta})\dd y,
\end{aligned}
\label{eq:orth-response-derivative}
\end{equation}
and define
\begin{equation}
\dot m_{t+\tau,u}^{g_{\bar a},\bar a}(\bar H_{t+\tau})
:=
\dot\ell_{g_{\bar a}}(Y_{t+\tau},\bar H_t).
\label{eq:orth-terminal-derivative}
\end{equation}
At the true nuisance, write
\begin{equation}
\dot m_{t+\delta}^{g_{\bar a},\bar a}
:=
\left.\dot m_{t+\delta,u}^{g_{\bar a},\bar a}\right|_{u=0},
\qquad \delta=0,\ldots,\tau.
\end{equation}

At the true continuation nuisance, Eqs.~\ref{eq:terminal-stage-xi}
and \ref{eq:xi-recursion} imply, for every \(s\),
\begin{equation}
m_{t+\delta,0}^{g_{\bar a,s},\bar a}(\bar H_{t+\delta})
=
\mathbb E_0\!\left[
m_{t+\delta+1,0}^{g_{\bar a,s},\bar a}(\bar H_{t+\delta+1})
\mid \bar H_{t+\delta},A_{t+\delta}=a_{t+\delta}
\right],
\qquad
\delta=0,\ldots,\tau-1.
\end{equation}
Differentiating at \(s=0\) gives
\begin{equation}
\dot m_{t+\delta}^{g_{\bar a},\bar a}(\bar H_{t+\delta})
=
\mathbb E_0\!\left[
\dot m_{t+\delta+1}^{g_{\bar a},\bar a}(\bar H_{t+\delta+1})
\mid
\bar H_{t+\delta},A_{t+\delta}=a_{t+\delta}
\right],
\qquad
\delta=0,\ldots,\tau-1.
\label{eq:orth-derivative-recursion}
\end{equation}

Define the mixed derivative by
\begin{equation}
\begin{aligned}
&D_\eta D_{g_{\bar a}}
\mathcal L_{\mathrm{OrthoGen}}^{\bar a}
(g_{\bar a}^\star;\eta)
\left[
g_{\bar a}-g_{\bar a}^\star,
\widetilde\eta-\eta
\right]
:=
\left.
\frac{d}{du}
\left\{
\left.
\frac{d}{ds}
\mathcal L_{\mathrm{OrthoGen}}^{\bar a}
(g_{\bar a,s};\eta_u)
\right|_{s=0}
\right\}
\right|_{u=0}.
\end{aligned}
\label{eq:orth-mixed-derivative}
\end{equation}

\orthogentheorem*

\begin{proof}
Differentiating Eq.~\ref{eq:orth-population-objective} with respect to the
target-model path gives
\begin{equation}
\begin{aligned}
&D_{g_{\bar a}}
\mathcal L_{\mathrm{OrthoGen}}^{\bar a}
(g_{\bar a}^\star;\eta_u)
\left[
g_{\bar a}-g_{\bar a}^\star
\right] =
\mathbb E_0\Bigg[
R_{u,t:t+\tau-1}^{\bar a}
\dot\ell_{g_{\bar a}}(Y_{t+\tau},\bar H_t)
+
\sum_{\delta=0}^{\tau-1}
R_{u,t:t+\delta-1}^{\bar a}
\left(
1-r_{u,t+\delta}^{\bar a}
\right)
\dot m_{t+\delta,u}^{g_{\bar a},\bar a}(\bar H_{t+\delta})
\Bigg].
\end{aligned}
\label{eq:orth-target-gradient}
\end{equation}

Since
\begin{equation}
R_{u,t:t+\delta-1}^{\bar a}-R_{u,t:t+\delta}^{\bar a}
=
R_{u,t:t+\delta-1}^{\bar a}\left(1-r_{u,t+\delta}^{\bar a}\right),
\label{eq:orth-weight-difference}
\end{equation}
Eq.~\ref{eq:orth-target-gradient} becomes
\begin{equation}
\begin{aligned}
&D_{g_{\bar a}}
\mathcal L_{\mathrm{OrthoGen}}^{\bar a}
(g_{\bar a}^\star;\eta_u)
\left[
g_{\bar a}-g_{\bar a}^\star
\right]
\\
&=
\mathbb E_0\Bigg[
\dot m_{t,u}^{g_{\bar a},\bar a}(\bar H_t)
+
\sum_{\delta=0}^{\tau-1}
R_{u,t:t+\delta}^{\bar a}
\left\{
\dot m_{t+\delta+1,u}^{g_{\bar a},\bar a}(\bar H_{t+\delta+1})
-
\dot m_{t+\delta,u}^{g_{\bar a},\bar a}(\bar H_{t+\delta})
\right\}
\Bigg].
\end{aligned}
\label{eq:orth-target-gradient-telescope}
\end{equation}

Set
\begin{equation}
\dot\Delta_{t+\delta}^{g_{\bar a}}
:=
\dot m_{t+\delta+1}^{g_{\bar a},\bar a}(\bar H_{t+\delta+1})
-
\dot m_{t+\delta}^{g_{\bar a},\bar a}(\bar H_{t+\delta}).
\label{eq:orth-residual}
\end{equation}
By Eq.~\ref{eq:orth-derivative-recursion},
\begin{equation}
\mathbb E_0\!\left[
\dot\Delta_{t+\delta}^{g_{\bar a}}
\mid
\bar H_{t+\delta},A_{t+\delta}=a_{t+\delta}
\right]
=
0,
\qquad
\delta=0,\ldots,\tau-1.
\label{eq:orth-residual-mean-zero}
\end{equation}

\textbf{Continuation-density derivative.}
Define
\begin{equation}
\begin{aligned}
\dot m_{\xi,t+\delta}^{g_{\bar a},\bar a}(\bar H_{t+\delta})
&:=
\left.
\frac{d}{du}
\dot m_{t+\delta,u}^{g_{\bar a},\bar a}(\bar H_{t+\delta})
\right|_{u=0}
=
\int_{\mathcal Y}
\dot\ell_{g_{\bar a}}(y,\bar H_t)\,
\left\{
\widetilde\xi_{t+\delta}^{\bar a}(y\mid\bar H_{t+\delta})
-
\xi_{t+\delta}^{\bar a}(y\mid\bar H_{t+\delta})
\right\}\dd y,
\end{aligned}
\label{eq:orth-xi-response-perturbation}
\end{equation}
Holding \(\pi\) fixed in Eq.~\ref{eq:orth-target-gradient} gives
\begin{equation}
\begin{aligned}
&D_\xi D_{g_{\bar a}}
\mathcal L_{\mathrm{OrthoGen}}^{\bar a}
(g_{\bar a}^\star;\eta)
\left[
g_{\bar a}-g_{\bar a}^\star,
\widetilde\xi-\xi
\right]
=
\sum_{\delta=0}^{\tau-1}
\mathbb E_0\!\left[
R_{t:t+\delta-1}^{\bar a}
\left(
1-r_{t+\delta}^{\bar a}
\right)
\dot m_{\xi,t+\delta}^{g_{\bar a},\bar a}(\bar H_{t+\delta})
\right].
\end{aligned}
\label{eq:orth-xi-cross-derivative}
\end{equation}

Equation~\ref{eq:proof-basic-conditional-weighting}, with \(Z=1\), gives
\(\mathbb E_0[r_{t+\delta}^{\bar a}\mid\bar H_{t+\delta}]=1\).
Because
\(R_{t:t+\delta-1}^{\bar a}
\dot m_{\xi,t+\delta}^{g_{\bar a},\bar a}(\bar H_{t+\delta})\) is
\(\bar H_{t+\delta}\)-measurable, every summand in
Eq.~\ref{eq:orth-xi-cross-derivative} is zero. Therefore,
\begin{equation}
D_\xi D_{g_{\bar a}}
\mathcal L_{\mathrm{OrthoGen}}^{\bar a}
(g_{\bar a}^\star;\eta)
\left[
g_{\bar a}-g_{\bar a}^\star,
\widetilde\xi-\xi
\right]
=
0.
\label{eq:orth-xi-zero}
\end{equation}

\textbf{Propensity derivative.}
For every \(\delta'=0,\ldots,\tau-1\),
\begin{equation}
\left.
\frac{d}{du}
r_{u,t+\delta'}^{\bar a}
\right|_{u=0}
=
-r_{t+\delta'}^{\bar a}
\frac{
\widetilde\pi_{t+\delta'}(a_{t+\delta'}\mid\bar H_{t+\delta'})
-
\pi_{t+\delta'}(a_{t+\delta'}\mid\bar H_{t+\delta'})
}{
\pi_{t+\delta'}(a_{t+\delta'}\mid\bar H_{t+\delta'})
}.
\end{equation}
The product rule gives
\begin{equation}
\left.
\frac{d}{du}R_{u,t:t+\delta}^{\bar a}
\right|_{u=0}
=
-R_{t:t+\delta}^{\bar a}
\sum_{\delta'=0}^{\delta}
\frac{
\widetilde\pi_{t+\delta'}(a_{t+\delta'}\mid\bar H_{t+\delta'})
-
\pi_{t+\delta'}(a_{t+\delta'}\mid\bar H_{t+\delta'})
}{
\pi_{t+\delta'}(a_{t+\delta'}\mid\bar H_{t+\delta'})
}.
\label{eq:orth-cumulative-weight-derivative}
\end{equation}

Holding \(\xi\) fixed in Eq.~\ref{eq:orth-target-gradient-telescope} gives
\begin{equation}
\begin{aligned}
&D_\pi D_{g_{\bar a}}
\mathcal L_{\mathrm{OrthoGen}}^{\bar a}
(g_{\bar a}^\star;\eta)
\left[
g_{\bar a}-g_{\bar a}^\star,
\widetilde\pi-\pi
\right]
=
-
\sum_{\delta=0}^{\tau-1}
\sum_{\delta'=0}^{\delta}
\mathbb E_0\!\left[
R_{t:t+\delta}^{\bar a}
\frac{
\widetilde\pi_{t+\delta'}(a_{t+\delta'}\mid\bar H_{t+\delta'})
-
\pi_{t+\delta'}(a_{t+\delta'}\mid\bar H_{t+\delta'})
}{
\pi_{t+\delta'}(a_{t+\delta'}\mid\bar H_{t+\delta'})
}
\dot\Delta_{t+\delta}^{g_{\bar a}}
\right].
\end{aligned}
\label{eq:orth-pi-cross-derivative}
\end{equation}

For \(\delta'\leq\delta\), the prefactor
\begin{equation}
R_{t:t+\delta-1}^{\bar a}
\frac{
\widetilde\pi_{t+\delta'}(a_{t+\delta'}\mid\bar H_{t+\delta'})
-
\pi_{t+\delta'}(a_{t+\delta'}\mid\bar H_{t+\delta'})
}{
\pi_{t+\delta'}(a_{t+\delta'}\mid\bar H_{t+\delta'})
}
\end{equation}
is \(\bar H_{t+\delta}\)-measurable. The IPW identity in
Eq.~\ref{eq:proof-basic-conditional-weighting} and
Eq.~\ref{eq:orth-residual-mean-zero} give
\begin{equation}
\begin{aligned}
\mathbb E_0\!\left[
r_{t+\delta}^{\bar a}
\dot\Delta_{t+\delta}^{g_{\bar a}}
\mid
\bar H_{t+\delta}
\right]
&=
\mathbb E_0\!\left[
\dot\Delta_{t+\delta}^{g_{\bar a}}
\mid
\bar H_{t+\delta},A_{t+\delta}=a_{t+\delta}
\right] = 0.
\end{aligned}
\label{eq:orth-weighted-residual-zero}
\end{equation}
Conditioning on \(\bar H_{t+\delta}\) therefore makes every summand in
Eq.~\ref{eq:orth-pi-cross-derivative} zero, and
\begin{equation}
D_\pi D_{g_{\bar a}}
\mathcal L_{\mathrm{OrthoGen}}^{\bar a}
(g_{\bar a}^\star;\eta)
\left[
g_{\bar a}-g_{\bar a}^\star,
\widetilde\pi-\pi
\right]
=
0.
\label{eq:orth-pi-zero}
\end{equation}

Linearity in the nuisance direction and
Eqs.~\ref{eq:orth-xi-zero}--\ref{eq:orth-pi-zero} yield
\begin{equation}
D_\eta D_{g_{\bar a}}
\mathcal L_{\mathrm{OrthoGen}}^{\bar a}
(g_{\bar a}^\star;\eta)
\left[
g_{\bar a}-g_{\bar a}^\star,
\widetilde\eta-\eta
\right]
=
0.
\end{equation}
\end{proof}

\subsection{Double robustness and quasi-oracle efficiency}
\label{sec:double_robustness_quasi_oracle}

Let
$\widehat\eta=(\widehat\pi,\widehat\xi)$ be fitted on an auxiliary sample or
complementary cross-fitting folds. We condition on the nuisance-training
sample throughout. The corresponding weights $\widehat r_{t+\delta}^{\bar a}$ and
$\widehat R_{t:t+\delta}^{\bar a}$ are obtained by replacing $\pi$ with
$\widehat\pi$ in Eq.~\ref{eq:orth-path-weights}.

For fixed $g_{\bar a}\in\cG$, let
$m_{t+\delta}^{g_{\bar a},\bar a}$ and
$\widehat m_{t+\delta}^{g_{\bar a},\bar a}$ denote the score responses in
Eq.~\ref{eq:ra-score-response} evaluated at $\xi_{t+\delta}^{\bar a}$ and
$\widehat\xi_{t+\delta}^{\bar a}$, respectively, and set
$m_{t+\tau}^{g_{\bar a},\bar a}(\bar H_{t+\tau})
:=\ell_{g_{\bar a}}(Y_{t+\tau},\bar H_t)$ as in
Eq.~\ref{eq:orth-terminal-response}.
The continuation recursion implies
\begin{align}
\mathcal L_{\bar a}(g_{\bar a})
&=
\mathbb E_0[m_t^{g_{\bar a},\bar a}(\bar H_t)],
\nonumber\\
m_{t+\delta}^{g_{\bar a},\bar a}(\bar H_{t+\delta})
&=
\mathbb E_0\!\left[
m_{t+\delta+1}^{g_{\bar a},\bar a}(\bar H_{t+\delta+1})
\,\middle|\,
\bar H_{t+\delta},A_{t+\delta}=a_{t+\delta}
\right],
\qquad \delta=0,\ldots,\tau-1.
\label{eq:qo-score-response-recursion}
\end{align}
Let $\mathcal L_{\mathrm{OrthoGen}}^{\bar a}
(g_{\bar a};\widehat\eta)$ denote the population objective in
Eq.~\ref{eq:orth-population-objective} evaluated at $\widehat\eta$.

\begin{lemma}[Exact longitudinal doubly robust remainder]
\label{lem:exact-longitudinal-dr-remainder}
Suppose the causal assumptions in Sec.~\ref{sec:problem-setting} hold,
$\widehat\pi_{t+\delta}(a_{t+\delta}\mid\bar H_{t+\delta})>0$ almost surely for $\delta=0,\ldots,\tau-1$, and the
displayed quantities are integrable. Then, for every fixed
$g_{\bar a}\in\cG$,
\begin{align}
&\mathcal L_{\mathrm{OrthoGen}}^{\bar a}
(g_{\bar a};\widehat\eta)
-
\mathcal L_{\bar a}(g_{\bar a})
\nonumber\\
&\quad=
\sum_{\delta=0}^{\tau-1}
\mathbb E_0\!\left[
\widehat R_{t:t+\delta-1}^{\bar a}
\frac{
\widehat\pi_{t+\delta}(a_{t+\delta}\mid\bar H_{t+\delta})-\pi_{t+\delta}(a_{t+\delta}\mid\bar H_{t+\delta})
}{
\widehat\pi_{t+\delta}(a_{t+\delta}\mid\bar H_{t+\delta})
}
\left\{
\widehat m_{t+\delta}^{g_{\bar a},\bar a}(\bar H_{t+\delta})
-
m_{t+\delta}^{g_{\bar a},\bar a}(\bar H_{t+\delta})
\right\}
\right].
\label{eq:exact-longitudinal-dr-remainder}
\end{align}
\end{lemma}

\begin{proof}
For any $\{m_{t+\delta}\}_{\delta=0}^{\tau}$,
\begin{align}
&\widehat R_{t:t+\tau-1}^{\bar a}m_{t+\tau}
+
\sum_{\delta=0}^{\tau-1}
\widehat R_{t:t+\delta-1}^{\bar a}
(1-\widehat r_{t+\delta}^{\bar a})m_{t+\delta}
=
m_t
+
\sum_{\delta=0}^{\tau-1}
\widehat R_{t:t+\delta}^{\bar a}(m_{t+\delta+1}-m_{t+\delta}).
\end{align}
Applying this identity to the sequence
$\{m_{t+\delta}^{g_{\bar a},\bar a}\}_{\delta=0}^{\tau}$, the recursion in
Eq.~\ref{eq:qo-score-response-recursion} and the conditional
inverse-probability identity give
\begin{align}
&\mathbb E_0\!\left[
\widehat r_{t+\delta}^{\bar a}
\left\{
m_{t+\delta+1}^{g_{\bar a},\bar a}(\bar H_{t+\delta+1})
-
m_{t+\delta}^{g_{\bar a},\bar a}(\bar H_{t+\delta})
\right\}
\,\middle|\,
\bar H_{t+\delta}
\right]
\\
&\quad=
\frac{\pi_{t+\delta}(a_{t+\delta}\mid\bar H_{t+\delta})}{\widehat\pi_{t+\delta}(a_{t+\delta}\mid\bar H_{t+\delta})}
\mathbb E_0\!\left[
m_{t+\delta+1}^{g_{\bar a},\bar a}(\bar H_{t+\delta+1})
-
m_{t+\delta}^{g_{\bar a},\bar a}(\bar H_{t+\delta})
\,\middle|\,
\bar H_{t+\delta},A_{t+\delta}=a_{t+\delta}
\right]
=0.
\end{align}
Since $\widehat R_{t:t+\delta-1}^{\bar a}$ is
$\bar H_{t+\delta}$-measurable,
\begin{equation}
\mathcal L_{\mathrm{OrthoGen}}^{\bar a}
(g_{\bar a};(\widehat\pi,\xi))
=
\mathbb E_0\!\left[
m_t^{g_{\bar a},\bar a}(\bar H_t)
\right]
=
\mathcal L_{\bar a}(g_{\bar a}).
\end{equation}

Subtracting the two objectives gives
\begin{align}
&\mathcal L_{\mathrm{OrthoGen}}^{\bar a}
(g_{\bar a};\widehat\eta)
-
\mathcal L_{\mathrm{OrthoGen}}^{\bar a}
(g_{\bar a};(\widehat\pi,\xi)) =
\sum_{\delta=0}^{\tau-1}
\mathbb E_0\!\left[
\widehat R_{t:t+\delta-1}^{\bar a}
(1-\widehat r_{t+\delta}^{\bar a})
\left\{
\widehat m_{t+\delta}^{g_{\bar a},\bar a}(\bar H_{t+\delta})
-
m_{t+\delta}^{g_{\bar a},\bar a}(\bar H_{t+\delta})
\right\}
\right].
\end{align}
The factor multiplying $1-\widehat r_{t+\delta}^{\bar a}$ is
$\bar H_{t+\delta}$-measurable,
and
\begin{equation}
\mathbb E_0\!\left[
1-\widehat r_{t+\delta}^{\bar a}
\,\middle|\,
\bar H_{t+\delta}
\right]
=
1-
\frac{\pi_{t+\delta}(a_{t+\delta}\mid\bar H_{t+\delta})}{\widehat\pi_{t+\delta}(a_{t+\delta}\mid\bar H_{t+\delta})}
=
\frac{\widehat\pi_{t+\delta}(a_{t+\delta}\mid\bar H_{t+\delta})-\pi_{t+\delta}(a_{t+\delta}\mid\bar H_{t+\delta})}
{\widehat\pi_{t+\delta}(a_{t+\delta}\mid\bar H_{t+\delta})}.
\end{equation}
Conditioning each summand on $\bar H_{t+\delta}$ proves
Eq.~\ref{eq:exact-longitudinal-dr-remainder}. No property of the score beyond
integrability is used.
\end{proof}

\begin{corollary}[Exact stagewise double robustness]
\label{cor:stagewise-double-robustness}
For every fixed $g_{\bar a}\in\cG$,
$\mathcal L_{\mathrm{OrthoGen}}^{\bar a}(g_{\bar a};\widehat\eta)
=
\mathcal L_{\bar a}(g_{\bar a})$ if, at every $\delta=0,\ldots,\tau-1$, either
\begin{equation}
\widehat\pi_{t+\delta}(a_{t+\delta}\mid\bar H_{t+\delta})
=
\pi_{t+\delta}(a_{t+\delta}\mid\bar H_{t+\delta})
\quad\text{or}\quad
\widehat\xi_{t+\delta}^{\bar a}(\cdot\mid\bar H_{t+\delta})
=
\xi_{t+\delta}^{\bar a}(\cdot\mid\bar H_{t+\delta})
\qquad \mathbb P_0\text{-almost surely}.
\end{equation}
The choice may vary across stages. More generally, for a fixed
$g_{\bar a}$, equality of the corresponding score responses
$\widehat m_{t+\delta}^{g_{\bar a},\bar a}=m_{t+\delta}^{g_{\bar a},\bar a}$ is sufficient
even when the full continuation densities differ.
\end{corollary}

\begin{proof}
Immediate from Eq.~\ref{eq:exact-longitudinal-dr-remainder}; equality of the
continuation densities implies equality of their score integrals.
\end{proof}

Thus, Corollary~\ref{cor:stagewise-double-robustness} is an exact, rather
than only rate, double-robustness statement for the population objective.

Let $g_{\bar a}^{\star}$ retain its definition as an oracle projection, let
$\widehat g_{\bar a}\in\cG$ denote the fitted second-stage target generator obtained using the estimated nuisances, and write
$\|\widehat g_{\bar a}-g_{\bar a}^{\star}\|_{\cG}$ for the relevant
model-dependent discrepancy. For measurable $f(y,\bar H_{t+\delta})$ at a
relative treatment step $\delta$, define
\begin{equation}
\|f\|_{4,2}
:=
\left[
\mathbb E_0\!\left\{
\left(
\int_{\mathcal Y}
f(y,\bar H_{t+\delta})^2\dd y
\right)^2
\right\}
\right]^{1/4}.
\end{equation}

\begin{assumption}[Target-model regularity]
\label{ass:target-model-regularity}
For the selected pair
$(g_{\bar a}^{\star},\widehat g_{\bar a})$, the following conditions hold.
\begin{enumerate}
\renewcommand{\theenumi}{(\roman{enumi})}
\renewcommand{\labelenumi}{\theenumi}
\item\label{ass:target-model-path}
There is an admissible path $s\mapsto g_{\bar a,s}$, $s\in[0,1]$, with
$g_{\bar a,0}=g_{\bar a}^{\star}$ and
$g_{\bar a,1}=\widehat g_{\bar a}$. The map
$s\mapsto\mathcal L_{\mathrm{OrthoGen}}^{\bar a}
(g_{\bar a,s};\widehat\eta)$ is twice continuously differentiable, while
$s\mapsto\mathcal L_{\bar a}(g_{\bar a,s})$ is right differentiable at zero;
differentiation may be interchanged with expectation and integration.
First- and second-order derivative notation with bracketed direction
$\widehat g_{\bar a}-g_{\bar a}^{\star}$ is induced by this path.

\item\label{ass:target-model-directional-optimality}
Moving from the oracle projection toward the fitted target cannot increase the population objective to first order:
\begin{equation}
D_g
\mathcal L_{\bar a}(g_{\bar a}^{\star})
[\widehat g_{\bar a}-g_{\bar a}^{\star}]
\leq 0.
\label{eq:qo-directional-optimality}
\end{equation}

\item\label{ass:target-model-curvature}
There is a constant $\alpha_{\cG}>0$ such that, at the intermediate model
$\bar g_{\bar a}=g_{\bar a,\bar s}$ furnished by the Taylor expansion,
\begin{equation}
-
D_g^2
\mathcal L_{\mathrm{OrthoGen}}^{\bar a}
(\bar g_{\bar a};\widehat\eta)
\left[
\widehat g_{\bar a}-g_{\bar a}^{\star},
\widehat g_{\bar a}-g_{\bar a}^{\star}
\right]
\geq
\alpha_{\cG}
\|\widehat g_{\bar a}-g_{\bar a}^{\star}\|_{\cG}^2.
\label{eq:qo-target-curvature}
\end{equation}

\item\label{ass:target-model-score-derivative}
There is a finite constant $B_{\cG}$ such that
\begin{equation}
\left\|
\left[
\int_{\mathcal Y}
\left|
D_g
\ell_{g_{\bar a}^{\star}}(y,\bar H_t)
[\widehat g_{\bar a}-g_{\bar a}^{\star}]
\right|^2\dd y
\right]^{1/2}
\right\|_{L^4(\mathbb P_0)}
\leq
B_{\cG}
\|\widehat g_{\bar a}-g_{\bar a}^{\star}\|_{\cG}.
\label{eq:qo-score-derivative-regularity}
\end{equation}
\end{enumerate}
\end{assumption}

\textbf{Interpretation of Assumption~\ref{ass:target-model-regularity}.} In our work, we slightly generalize the target model regularity assumptions, initially suggested by \citet{melnychuk2025gdr} for the cross-sectional setting. Still, our conditions~\ref{ass:target-model-path}--\ref{ass:target-model-score-derivative} are mild and can be realistically satisfied by a broad class of probabilistic models (including normalizing flows and diffusion models).
\begin{itemize}[leftmargin=0.5cm,itemsep=-0.45mm]
    \item First, our condition~\ref{ass:target-model-path} requires smoothness only along the selected path connecting $g_{\bar a}^{\star}$ and $\widehat g_{\bar a}$; and such pathwise differentiability and interchange of differentiation with integration are standard for smoothly parameterized generative objectives. 
    \item Further, condition~\ref{ass:target-model-directional-optimality} is automatic whenever the path remains locally in $\cG$, by the optimality of the oracle projection.
    \item  Then, in condition~\ref{ass:target-model-curvature}, curvature is required only in the realized direction $\widehat g_{\bar a}-g_{\bar a}^{\star}$ and only at the intermediate model arising in the one-dimensional Taylor expansion.\footnote{Closely related model-dependent curvature conditions for normalizing-flow and diffusion-based doubly robust generative learners have also been formulated and shown to hold asymptotically under regularity and nuisance-consistency conditions \citep{melnychuk2025gdr}.} This requirement is also compatible with standard generative models:
    \begin{itemize}[leftmargin=0.5cm,itemsep=-0.55mm]
        \item For normalizing flows, Eq.~\ref{eq:nf-risk-kl-decomposition} shows that population maximum likelihood is equivalent, up to an additive constant, to minimizing a forward Kullback--Leibler criterion \citep{papamakarios2021normalizing}, making local quadratic curvature natural under regularity and local identification.
        \item For diffusion models, the ELBO terms reduce to weighted noise-prediction errors, while Eq.~\ref{eq:dm-denoising-loss} gives the usual simplified squared-error objective \citep{ho2020denoising}; the criterion is therefore directly quadratic in the predicted-noise function.
    \end{itemize} 
    Thus, condition~\ref{ass:target-model-curvature} concerns curvature in the induced function-space discrepancy and does not require global convexity of the neural-network parameterization.
    \item Finally, condition~\ref{ass:target-model-score-derivative} is a standard local continuity and moment condition: It requires directional changes in the generative score to be controlled by the chosen model discrepancy and to possess a suitable fourth-moment envelope. 
\end{itemize}


To obtain the propensity-error form in Theorem~\ref{thm:quasi-oracle}, we
also require estimated positivity: for some fixed
$\varepsilon\in(0,1]$ and every $\delta=0,\ldots,\tau-1$,
\begin{equation}
\widehat\pi_{t+\delta}(a_{t+\delta}\mid\bar H_{t+\delta})\geq\varepsilon
\qquad \mathbb P_0\text{-almost surely}.
\label{eq:qo-estimated-positivity}
\end{equation}

The estimated-positivity condition in Eq.~\ref{eq:qo-estimated-positivity} is the usual stability requirement for inverse-propensity weighting and can be enforced in practice by clipping estimated propensities away from zero.

In the following, whenever constants depending on the horizon are absorbed into $\lesssim$,
$\tau$ is treated as fixed.

\quasioracletheorem*

\begin{proof}[Proof of Theorem~\ref{thm:quasi-oracle}]
Recall the conditional-population optimization error from the main text,
written in the notation of Eq.~\ref{eq:orth-population-objective} as
\begin{equation}
\mathrm{Opt}_{\mathcal G}^{\bar a}
:=
\sup_{g_{\bar a}\in\mathcal G}
\mathcal L_{\mathrm{OrthoGen}}^{\bar a}
(g_{\bar a};\widehat\eta)
-
\mathcal L_{\mathrm{OrthoGen}}^{\bar a}
(\widehat g_{\bar a};\widehat\eta).
\label{eq:qo-optimization-error}
\end{equation}
Write
\begin{equation}
W_{t+\delta}^{\bar a}
:=
\widehat R_{t:t+\delta-1}^{\bar a}
\frac{
\widehat\pi_{t+\delta}(a_{t+\delta}\mid\bar H_{t+\delta})-\pi_{t+\delta}(a_{t+\delta}\mid\bar H_{t+\delta})
}{
\widehat\pi_{t+\delta}(a_{t+\delta}\mid\bar H_{t+\delta})
}.
\end{equation}
A second-order Taylor expansion along the selected path gives, for some
intermediate $\bar g_{\bar a}$,
\begin{align}
&\mathcal L_{\mathrm{OrthoGen}}^{\bar a}
(\widehat g_{\bar a};\widehat\eta)
\nonumber =
\mathcal L_{\mathrm{OrthoGen}}^{\bar a}
(g_{\bar a}^{\star};\widehat\eta)
+
D_g\mathcal L_{\mathrm{OrthoGen}}^{\bar a}
(g_{\bar a}^{\star};\widehat\eta)
[\widehat g_{\bar a}-g_{\bar a}^{\star}]
\nonumber\\
&\qquad+
\frac{1}{2}
D_g^2\mathcal L_{\mathrm{OrthoGen}}^{\bar a}
(\bar g_{\bar a};\widehat\eta)
\left[
\widehat g_{\bar a}-g_{\bar a}^{\star},
\widehat g_{\bar a}-g_{\bar a}^{\star}
\right].
\label{eq:qo-taylor-expansion}
\end{align}
Condition~\ref{ass:target-model-curvature} and
$g_{\bar a}^{\star}\in\cG$ imply
\begin{align}
\frac{\alpha_{\cG}}{2}
\|\widehat g_{\bar a}-g_{\bar a}^{\star}\|_{\cG}^2
&\leq
\mathrm{Opt}_{\mathcal G}^{\bar a}
+
D_g\mathcal L_{\mathrm{OrthoGen}}^{\bar a}
(g_{\bar a}^{\star};\widehat\eta)
[\widehat g_{\bar a}-g_{\bar a}^{\star}].
\label{eq:qo-before-gradient-control}
\end{align}
Here the first term is bounded by
$\mathrm{Opt}_{\mathcal G}^{\bar a}$ because
$g_{\bar a}^{\star}\in\mathcal G$. To isolate the nuisance perturbation,
add and subtract the true-nuisance population objective:
\begin{equation}
\mathcal L_{\mathrm{OrthoGen}}^{\bar a}(g_{\bar a};\widehat\eta)
=
\mathcal L_{\mathrm{OrthoGen}}^{\bar a}(g_{\bar a};\eta)
+
\left\{
\mathcal L_{\mathrm{OrthoGen}}^{\bar a}(g_{\bar a};\widehat\eta)
-
\mathcal L_{\mathrm{OrthoGen}}^{\bar a}(g_{\bar a};\eta)
\right\}.
\end{equation}
Lemma~\ref{lem:exact-longitudinal-dr-remainder} gives
$\mathcal L_{\mathrm{OrthoGen}}^{\bar a}(g_{\bar a};\eta)
=\mathcal L_{\bar a}(g_{\bar a})$. Hence,
condition~\ref{ass:target-model-directional-optimality} yields
\begin{align}
&D_g\mathcal L_{\mathrm{OrthoGen}}^{\bar a}
(g_{\bar a}^{\star};\widehat\eta)
[\widehat g_{\bar a}-g_{\bar a}^{\star}]
\leq
\left|
D_g\mathcal L_{\mathrm{OrthoGen}}^{\bar a}
(g_{\bar a}^{\star};\widehat\eta)
[\widehat g_{\bar a}-g_{\bar a}^{\star}]
\right.\left.
-
D_g\mathcal L_{\mathrm{OrthoGen}}^{\bar a}
(g_{\bar a}^{\star};\eta)
[\widehat g_{\bar a}-g_{\bar a}^{\star}]
\right|.
\label{eq:qo-gradient-comparison}
\end{align}

Differentiating Eq.~\ref{eq:exact-longitudinal-dr-remainder} in the same
direction yields
\begin{align}
&D_g
\mathcal L_{\mathrm{OrthoGen}}^{\bar a}
(g_{\bar a}^{\star};\widehat\eta)
[\widehat g_{\bar a}-g_{\bar a}^{\star}]
-
D_g
\mathcal L_{\mathrm{OrthoGen}}^{\bar a}
(g_{\bar a}^{\star};\eta)
[\widehat g_{\bar a}-g_{\bar a}^{\star}]
\nonumber\\
&\quad=
\sum_{\delta=0}^{\tau-1}
\mathbb E_0\!\left[
W_{t+\delta}^{\bar a}
\int_{\mathcal Y}
D_g\ell_{g_{\bar a}^{\star}}(y,\bar H_t)
[\widehat g_{\bar a}-g_{\bar a}^{\star}]
\left\{
\widehat\xi_{t+\delta}^{\bar a}(y\mid\bar H_{t+\delta})
-
\xi_{t+\delta}^{\bar a}(y\mid\bar H_{t+\delta})
\right\}\dd y
\right].
\label{eq:qo-gradient-remainder}
\end{align}
By Cauchy--Schwarz in $y$, H\"{o}lder's inequality with exponents $(2,4,4)$,
and condition~\ref{ass:target-model-score-derivative},
\begin{align}
&\left|
\mathbb E_0\!\left[
W_{t+\delta}^{\bar a}
\int_{\mathcal Y}
D_g\ell_{g_{\bar a}^{\star}}(y,\bar H_t)
[\widehat g_{\bar a}-g_{\bar a}^{\star}]
\left\{
\widehat\xi_{t+\delta}^{\bar a}(y\mid\bar H_{t+\delta})
-
\xi_{t+\delta}^{\bar a}(y\mid\bar H_{t+\delta})
\right\}\dd y
\right]
\right|
\\
&\quad\leq
\|W_{t+\delta}^{\bar a}\|_{L^2(\mathbb P_0)}
\left\|
\left[
\int_{\mathcal Y}
\left|
D_g\ell_{g_{\bar a}^{\star}}(y,\bar H_t)
[\widehat g_{\bar a}-g_{\bar a}^{\star}]
\right|^2\dd y
\right]^{1/2}
\right\|_{L^4(\mathbb P_0)}
\left\|
\widehat\xi_{t+\delta}^{\bar a}(\cdot\mid\bar H_{t+\delta})
-
\xi_{t+\delta}^{\bar a}(\cdot\mid\bar H_{t+\delta})
\right\|_{4,2}
\\
&\quad\leq
B_{\cG}
\|\widehat g_{\bar a}-g_{\bar a}^{\star}\|_{\cG}
\|W_{t+\delta}^{\bar a}\|_{L^2(\mathbb P_0)}
\left\|
\widehat\xi_{t+\delta}^{\bar a}(\cdot\mid\bar H_{t+\delta})
-
\xi_{t+\delta}^{\bar a}(\cdot\mid\bar H_{t+\delta})
\right\|_{4,2}.
\end{align}
Combining this bound with Eqs.~\ref{eq:qo-before-gradient-control} and
\ref{eq:qo-gradient-comparison} gives
\begin{align}
\frac{\alpha_{\cG}}{2}
\|\widehat g_{\bar a}-g_{\bar a}^{\star}\|_{\cG}^2
&\leq
\mathrm{Opt}_{\mathcal G}^{\bar a}
\nonumber\\
&\quad+
B_{\cG}
\|\widehat g_{\bar a}-g_{\bar a}^{\star}\|_{\cG}
\left[
\sum_{\delta=0}^{\tau-1}
\|W_{t+\delta}^{\bar a}\|_{L^2(\mathbb P_0)}
\left\|
\widehat\xi_{t+\delta}^{\bar a}(\cdot\mid\bar H_{t+\delta})
-
\xi_{t+\delta}^{\bar a}(\cdot\mid\bar H_{t+\delta})
\right\|_{4,2}
\right].
\label{eq:qo-gradient-bound}
\end{align}
Young's inequality yields
\begin{align}
&B_{\cG}
\|\widehat g_{\bar a}-g_{\bar a}^{\star}\|_{\cG}
\left[
\sum_{\delta=0}^{\tau-1}
\|W_{t+\delta}^{\bar a}\|_{L^2(\mathbb P_0)}
\left\|
\widehat\xi_{t+\delta}^{\bar a}(\cdot\mid\bar H_{t+\delta})
-
\xi_{t+\delta}^{\bar a}(\cdot\mid\bar H_{t+\delta})
\right\|_{4,2}
\right]
\\
&\quad\leq
\frac{\alpha_{\cG}}{4}
\|\widehat g_{\bar a}-g_{\bar a}^{\star}\|_{\cG}^2
+
\frac{B_{\cG}^2}{\alpha_{\cG}}
\left[
\sum_{\delta=0}^{\tau-1}
\|W_{t+\delta}^{\bar a}\|_{L^2(\mathbb P_0)}
\left\|
\widehat\xi_{t+\delta}^{\bar a}(\cdot\mid\bar H_{t+\delta})
-
\xi_{t+\delta}^{\bar a}(\cdot\mid\bar H_{t+\delta})
\right\|_{4,2}
\right]^2.
\end{align}
Rearranging yields the explicit intermediate bound
\begin{equation}
\begin{aligned}
\|\widehat g_{\bar a}-g_{\bar a}^{\star}\|_{\mathcal G}^2
&\leq
\frac{4}{\alpha_{\mathcal G}}
\mathrm{Opt}_{\mathcal G}^{\bar a} +
\frac{4B_{\mathcal G}^2}{\alpha_{\mathcal G}^2}
\left[
\sum_{\delta=0}^{\tau-1}
\|W_{t+\delta}^{\bar a}\|_{L^2(\mathbb P_0)}
\left\|
\widehat\xi_{t+\delta}^{\bar a}(\cdot\mid\bar H_{t+\delta})
-
\xi_{t+\delta}^{\bar a}(\cdot\mid\bar H_{t+\delta})
\right\|_{4,2}
\right]^2.
\end{aligned}
\label{eq:generic-quasi-oracle-bound}
\end{equation}

It remains to control the stagewise weight error. Under
Eq.~\ref{eq:qo-estimated-positivity}, each estimated one-step weight is at
most $\varepsilon^{-1}$, and therefore
\begin{equation}
0\leq \widehat R_{t:t+\delta-1}^{\bar a}
\leq \varepsilon^{-\delta},
\qquad
\|\widehat R_{t:t+\delta-1}^{\bar a}\|_{L^4(\mathbb P_0)}
\leq \varepsilon^{-\delta}.
\end{equation}
Thus, H\"{o}lder's inequality gives
\begin{align}
\|W_{t+\delta}^{\bar a}\|_{L^2(\mathbb P_0)}
&\leq
\varepsilon^{-1}
\|\widehat R_{t:t+\delta-1}^{\bar a}\|_{L^4(\mathbb P_0)}
\|\widehat\pi_{t+\delta}(a_{t+\delta}\mid\bar H_{t+\delta})-\pi_{t+\delta}(a_{t+\delta}\mid\bar H_{t+\delta})\|_{L^4(\mathbb P_0)}
\\
&\quad\leq
\varepsilon^{-(\delta+1)}
\|\widehat\pi_{t+\delta}(a_{t+\delta}\mid\bar H_{t+\delta})-\pi_{t+\delta}(a_{t+\delta}\mid\bar H_{t+\delta})\|_{L^4(\mathbb P_0)}.
\end{align}
In particular,
\begin{equation}
\|W_{t+\delta}^{\bar a}\|_{L^2(\mathbb P_0)}
\leq
C_{\delta}
\|\widehat\pi_{t+\delta}-\pi_{t+\delta}\|_{L^4(\mathbb P_0)},
\qquad
C_{\delta}:=\varepsilon^{-(\delta+1)},
\label{eq:qo-stagewise-weight-error}
\end{equation}
where the propensity functions in the norm are evaluated at
$(a_{t+\delta},\bar H_{t+\delta})$. Substitution into
Eq.~\ref{eq:generic-quasi-oracle-bound} gives
\begin{equation}
\begin{aligned}
\|\widehat g_{\bar a}-g_{\bar a}^{\star}\|_{\mathcal G}^2
&\leq
\frac{4}{\alpha_{\mathcal G}}
\mathrm{Opt}_{\mathcal G}^{\bar a} +
\frac{4B_{\mathcal G}^2}{\alpha_{\mathcal G}^2}
\left[
\sum_{\delta=0}^{\tau-1}
C_{\delta}
\|\widehat\pi_{t+\delta}-\pi_{t+\delta}\|_{L^4(\mathbb P_0)}
\|\widehat\xi_{t+\delta}^{\bar a}-\xi_{t+\delta}^{\bar a}\|_{4,2}
\right]^2.
\end{aligned}
\label{eq:generic-quasi-oracle-positivity-bound}
\end{equation}
The additional cumulative-weight condition used in the previous formulation
is therefore implied by estimated positivity and is not needed separately.

Finally, fixed $\tau$ gives
$\max_{0\leq\delta\leq\tau-1}C_{\delta}\leq\varepsilon^{-\tau}$. Absorbing
$\alpha_{\mathcal G}^{-1}$, $B_{\mathcal G}$,
$\varepsilon^{-1}$, and fixed-horizon factors into $\lesssim$ yields
\begin{equation}
\|\widehat g_{\bar a}-g_{\bar a}^{\star}\|_{\mathcal G}^2
\lesssim
\mathrm{Opt}_{\mathcal G}^{\bar a}
+
\left[
\sum_{\delta=0}^{\tau-1}
\|\widehat\pi_{t+\delta}-\pi_{t+\delta}\|_{L^4(\mathbb P_0)}
\|\widehat\xi_{t+\delta}^{\bar a}-\xi_{t+\delta}^{\bar a}\|_{4,2}
\right]^2,
\end{equation}
which is Eq.~\ref{eq:quasi_oracle_main}. True propensity positivity remains
part of the causal identification assumptions, but only estimated positivity
is used in the stagewise weight bound above.
\end{proof}

\paragraph{Rate interpretation.}
Suppose the preceding conditions hold with probability tending to one, with
$\alpha_{\mathcal G}$ bounded away from zero, $B_{\mathcal G}$ bounded
above, fixed $\varepsilon>0$, and fixed $\tau$. For every
$\delta=0,\ldots,\tau-1$, let
\begin{align}
\|\widehat\pi_{t+\delta}-\pi_{t+\delta}\|_{L^4(\mathbb P_0)}
&=O_{\mathbb P}(\varepsilon_{\pi,t+\delta,n}),
\\
\|\widehat\xi_{t+\delta}^{\bar a}-\xi_{t+\delta}^{\bar a}\|_{4,2}
&=O_{\mathbb P}(\varepsilon_{\xi,t+\delta,n}).
\end{align}
Matching the main-text convention, $\rho_n$ is a rate for the
target-generator norm, so the corresponding squared-scale condition is
$\mathrm{Opt}_{\mathcal G}^{\bar a}=O_{\mathbb P}(\rho_n^2)$.
Theorem~\ref{thm:quasi-oracle} then gives
\begin{equation}
\|\widehat g_{\bar a}-g_{\bar a}^{\star}\|_{\mathcal G}^2
=
O_{\mathbb P}\!\left(
\rho_n^2
+
\left[
\sum_{\delta=0}^{\tau-1}\varepsilon_{\pi,t+\delta,n}\varepsilon_{\xi,t+\delta,n}
\right]^2
\right).
\label{eq:generic-rate-double-robustness}
\end{equation}
Equivalently, the target-generator norm is
$O_{\mathbb P}(\rho_n+
\sum_{\delta=0}^{\tau-1}\varepsilon_{\pi,t+\delta,n}\varepsilon_{\xi,t+\delta,n})$.

\emph{Exact double robustness} is supplied by
Lemma~\ref{lem:exact-longitudinal-dr-remainder} and
Corollary~\ref{cor:stagewise-double-robustness}: at each stage, correctness
of either nuisance component makes the corresponding product remainder
vanish. \emph{Rate double robustness} follows because nuisance effects enter
through products, so a slower rate for one nuisance can be offset by a faster
rate for the other. Finally, if
\begin{equation}
\sum_{\delta=0}^{\tau-1}\varepsilon_{\pi,t+\delta,n}\varepsilon_{\xi,t+\delta,n}
=o(\rho_n),
\label{eq:generic-rate-negligibility}
\end{equation}
then nuisance estimation is asymptotically negligible at the target-learning
scale: the norm and squared norm are respectively
$O_{\mathbb P}(\rho_n)$ and $O_{\mathbb P}(\rho_n^2)$. This is the stated
\emph{quasi-oracle efficiency}, since the learner attains the rate of the
corresponding oracle learner with known nuisances.

\section{Dataset Details}
\label{app:dataset-details}

\subsection{Fully synthetic dataset}
\label{app:fully-synthetic-dataset}

\textbf{Overview.}
We use a fully synthetic benchmark to evaluate joint distributional
learning when the outcome is multivariate and the covariate process itself
responds to treatment. Each trajectory contains a scalar time-varying state
$X_t$, a binary treatment $A_t$, and a three-dimensional continuous outcome
$Y_t=(Y_{t1},Y_{t2},Y_{t3})^\top$. The data are
generated directly in the timing convention of Sec.~\ref{sec:problem-setting},
$\bar H_t\rightarrow A_t\rightarrow (X_{t+1},Y_{t+1})$. We simulate trajectories of length 20 and use disjoint
outer partitions of 3,000 training, 1,000 validation, and 1,000 test patients.
As in the semi-synthetic experiment, the training partition is divided evenly
between nuisance and target-model fitting. Each outcome coordinate is
standardized using its mean and standard deviation in the outer training
partition, and this transformation is kept fixed for validation, test, and
oracle samples. The equations below describe the raw DGP scale; the
coordinate-wise outcome transformation is applied only after complete
trajectories have been generated.

On the raw DGP scale, the baseline variables are initialized independently as
\begin{equation}
X_0\sim\mathcal N(0,0.8^2),
\qquad
Y_0\sim\mathcal N_3(0,D_0R_0D_0),
\qquad X_0\perp Y_0,
\label{eq:fully-synthetic-initial}
\end{equation}
where $D_0=\operatorname{diag}(0.60,0.66,0.57)$ and
\begin{equation}
R_0=
\begin{pmatrix}
1 & 0.32 & 0.20\\
0.32 & 1 & 0.28\\
0.20 & 0.28 & 1
\end{pmatrix}.
\end{equation}
This Gaussian law for $Y_0$ initializes the recursion only; the conditional
Gaussian/asymmetric-Laplace mixture in
Eq.~\ref{eq:fully-synthetic-outcome-mixture} governs each post-baseline
outcome $Y_{t+1}$ for $t=0,\ldots,18$. The complete numerical
configuration---including all coefficients, scales, clipping constants, and
correlation matrices---is provided in the code. We
describe the structure of the data-generating process below.

\textbf{State and treatment processes.}
The next state follows
\begin{equation}
X_{t+1}
=\alpha_X X_t+\beta_Ys^\top Y_t+\beta_A(A_t-0.5)
+\epsilon^X_{t+1},
\qquad
\epsilon^X_{t+1}\sim\mathcal N(0,\sigma_X^2),
\label{eq:fully-synthetic-state}
\end{equation}
where $s$ is a fixed weighted summary of the three outcomes. For treatment
assignment, $g_X$ and $g_Y$ denote clipped and scaled versions of the current
state and outcomes. The observational propensity is
\begin{equation}
P(A_t=1\mid\bar H_t)
=\operatorname{sigmoid}\!\left
\{\gamma_0+\gamma_Xg_X(X_t)+\gamma_Yr^\top g_Y(Y_t)
+\gamma_A(A_{t-1}-0.5)\right\},
\label{eq:fully-synthetic-treatment}
\end{equation}
where $r$ combines all three outcome coordinates and $A_{-1}=0.5$ is used at
the first decision. Clipping bounds the treatment logits and avoids exact-zero
or exact-one propensities. Since the current state and outcomes affect
treatment, while treatment subsequently affects both the next state and next
outcome, the process has explicit treatment--confounder feedback.

\textbf{Joint outcome law.}
Conditional on the current history and treatment, each outcome coordinate has
a Gaussian/asymmetric-Laplace mixture marginal,
\begin{equation}
Y_{t+1,j}\mid\bar H_t,A_t
\sim
w_{t,j}\mathcal N(\mu^G_{t,j},\sigma_{G,t,j}^2)
+(1-w_{t,j})\operatorname{ALD}(\mu^L_{t,j},b_{t,j},q_{t,j}),
\label{eq:fully-synthetic-outcome-mixture}
\end{equation}
where $\mu^L_{t,j}$ is the $q_{t,j}$ quantile of the asymmetric-Laplace
component. The Gaussian location contains a stable, near-diagonal vector
autoregression together with coordinate-specific sinusoidal state terms and
centered treatment terms. The asymmetric-Laplace location equals the Gaussian
location plus a coordinate-specific hyperbolic-tangent shift in the state.
Both component scales, the mixture probability $w_{t,j}$, and the asymmetry
parameter $q_{t,j}$ vary with the clipped state, the corresponding current
outcome coordinate, and the current treatment. The three coordinates are not
conditionally independent. Cross-coordinate dependence uses the same
configured correlation matrix for the multivariate Gaussian innovation and
for Gaussian-copula latent variables underlying each of the two exponential
vectors that form the asymmetric-Laplace draw and the mixture indicators.
This preserves the specified marginal mixtures while inducing contemporaneous
dependence across outcomes. All numerical parameters are given in the code.

\textbf{Interventional ground truth.}
For a requested regime $\barat$, the oracle fixes the factual history through
$t$ and then recursively simulates both future states and future outcomes. We impose the alternating sequence specified in the main
evaluation section; for example, $\barat=(1)$ when $\tau=1$,
$\barat=(1,0)$ when $\tau=2$, and $\barat=(1,0,1,0,1)$ when $\tau=5$.
For each horizon we set $t=19-\tau$, so all methods target the outcome at the
final state time. Evaluation uses 300 draws from the fitted model and 300
oracle draws for each test trajectory. We compare the resulting
three-dimensional empirical distributions using the squared 2-Wasserstein
transport cost $W_2^2$, averaged across test trajectories; the optimal
transport is computed on complete outcome vectors.

\subsection{Semi-synthetic MIMIC-III dataset}
\label{app:semisynthetic-mimic}

\textbf{Cohort and covariates.}
Our semi-synthetic benchmark follows the construction of
\citet{hess2026overlapweighted}: it combines real covariate trajectories from
MIMIC-III \citep{johnson2016mimic}, represented as hourly records from the
MIMIC-Extract pipeline \citep{wang2020mimic}, with simulated treatments and
outcomes. We draw a training
cohort of 1,500 patients for each stage of our sample-split procedure.
Specifically, from patients with at least 20 hourly records, we sample 5,000
without replacement and retain the first 20 hours of each trajectory. The
resulting patient-level partitions contain 3,000 training (1500 for each stage), 1,000 validation,
and 1,000 test trajectories.

The time-varying covariates are heart rate, red blood cell count, sodium, mean
blood pressure, systemic vascular resistance, glucose, urine chloride,
Glasgow Coma Scale total, hematocrit, set positive end-expiratory pressure,
and respiratory rate. Age and gender are included as static covariates. To
match the released data pipeline, the selected hourly vital and laboratory
channels are forward-filled and then backward-filled before patients are
sampled. We standardize the selected continuous channels, standardize age,
and one-hot encode gender. Static features are repeated over time and
concatenated with the eleven dynamic measurements when evaluating the
data-generating mechanism. After simulation, covariates and generated outcomes are
standardized using means and standard deviations from the outer training
partition; the same transformation is then applied to validation and test
patients. The whole-frame filling and selected-cohort preprocessing are kept
to reproduce the reference benchmark; all model fitting takes place only
after this construction step.

\textbf{Treatment and outcome mechanism.} Let $\widetilde X_t\in\mathbb R^d$ contain the eleven
measurements and the repeated encoded static covariates at hour $t$. We define
\begin{equation}
q_t=\frac{1}{d}\sum_{j=1}^d\sin(\widetilde X_{tj}),
\qquad
c_t=\frac{1}{d}\sum_{j=1}^d\cos(\widetilde X_{tj}).
\label{eq:semisynthetic-covariate-scores}
\end{equation}
The initial binary treatment $B_0$ and scalar outcome $\widetilde Y_0$ are
drawn as
\begin{equation}
B_0\sim\operatorname{Bernoulli}\!\left\{\operatorname{sigmoid}(q_0)\right\},
\qquad
\widetilde Y_0=2(B_0-1/2)+\epsilon_0,
\qquad
\epsilon_0\sim\mathcal N(0,0.1^2).
\label{eq:semisynthetic-initial-state}
\end{equation}
For $t=1,\ldots,19$, factual treatments are sampled according to
\begin{align}
\eta_t
&=
\sin(\widetilde Y_{t-1})+q_{t-1}-B_{t-1}+\nu_t,
& \nu_t&\sim\mathcal N(0,0.2^2),
\nonumber\\
B_t
&\sim\operatorname{Bernoulli}\!\left\{\operatorname{sigmoid}(\eta_t)\right\}.
\label{eq:semisynthetic-treatment}
\end{align}
Writing $D_t=2(B_t-1/2)\in\{-1,1\}$, the outcome then evolves as
\begin{equation}
\widetilde Y_t
=
\widetilde Y_{t-1}
+D_t\exp\!\left\{
\frac{D_t c_t}{t}\sin(\widetilde Y_{t-1})
\right\}
+\epsilon_t,
\qquad
\epsilon_t\sim\mathcal N(0,0.1^2).
\label{eq:semisynthetic-outcome}
\end{equation}
The innovations are drawn independently. Thus, previous outcomes affect later
treatment assignment, while treatment, current covariates, and the previous
outcome jointly determine the next outcome. 

\textbf{Timing, splitting, and ground truth.}
The reference generator assigns $B_t$ and $\widetilde Y_t$ at the same raw
hour. We align it with our pre-interventional history convention by setting
$A_t=B_{t+1}$ and retaining $B_0$ in the baseline history. For a horizon
$\tau$, we set $t=19-\tau$ and intervene with the all-treated sequence
$a_{t+\delta}=1$ for $\delta=0,\ldots,\tau-1$. Hence, the regimes for
$\tau=1,\ldots,5$ are $(1)$, $(1,1)$, $(1,1,1)$, $(1,1,1,1)$, and
$(1,1,1,1,1)$, respectively, and the terminal outcome is always
$\widetilde Y_{19}$.

The 3,000 training patients are split once, by patient, into 1,500 patients
for nuisance estimation and 1,500 for fitting the target generator;
validation and test patients remain separate. For evaluation, the oracle starts from each test
patient's realized outcome at the prediction origin, keeps that patient's
covariate path fixed, imposes the all-treated treatment sequence, and
resamples the subsequent outcome innovations. The reported squared
2-Wasserstein costs compare 100 samples from each fitted model with 100 such
oracle draws per test patient.

\subsection{Real-world MIMIC-III dataset}
\label{app:realworld-mimic}

\textbf{Cohort and covariates.}
Our real-world benchmark uses the same MIMIC-III cohort construction and
preprocessing pipeline as the semi-synthetic experiment, but keeps the factual
treatments and factual outcomes observed in the data. Concretely, we use hourly
records from MIMIC-III \citep{johnson2016mimic} extracted with the
MIMIC-Extract pipeline \citep{wang2020mimic}, retain patients with at least 20
hourly records, and keep the first 20 hours of each selected trajectory. The
dataset is split at the patient level into 3,000 training, 1,000 validation,
and 1,000 test trajectories, with the training partition further divided into
1,500 patients for nuisance estimation and 1,500 for fitting the final target
generator.

The covariates are exactly those used in the semi-synthetic benchmark: eleven
time-varying clinical measurements together with static age and gender.
Specifically, the dynamic channels are heart rate, red blood cell count,
sodium, mean blood pressure, systemic vascular resistance, glucose, urine
chloride, Glasgow Coma Scale total, hematocrit, set positive end-expiratory
pressure, and respiratory rate. As above, selected hourly channels are
forward-filled and backward-filled before patient sampling; continuous
covariates and outcomes are standardized using statistics from the outer
training partition, and gender is one-hot encoded.

\textbf{Treatment, outcome, and evaluation setup.}
Unlike the semi-synthetic benchmark, no simulated treatment or outcome model is
introduced here: the observed treatment path and observed outcome trajectory are
taken as the real-world data source. Specifically, the observed treatment $A_t$
is the binary mechanical-ventilation indicator (\texttt{vent}) at raw hourly
index $t+1$, whereas the observed outcome $Y_t$ is diastolic blood pressure at
raw hourly index $t$; this one-step reindexing aligns the data with our convention
that $A_t$ is applied after the history through time $t$. For evaluation horizons
$\tau\in\{1,3,5\}$, we set
$t=9-\tau$ and predict the distribution of the factual terminal outcome under
the all-treated sequence $a_{t+\delta}=1$ for $\delta=0,\ldots,\tau-1$. This
choice keeps the terminal evaluation time fixed at $Y_{t+\tau}=Y_9$ (hour~9)
across horizons. Thus, the
intervention sequences are $(1)$ for $\tau=1$, $(1,1,1)$ for $\tau=3$, and
$(1,1,1,1,1)$ for $\tau=5$. Because counterfactual ground truth is unavailable
in the real-world setting, the squared Wasserstein cost used for the synthetic
experiments is not available here: it requires paired empirical samples from a
known oracle distribution over complete trajectories, which we do not have.
Instead, we evaluate factual prediction quality directly against the observed
outcome with the continuous ranked probability score (CRPS), which is a proper
scoring rule for one-dimensional predictive distributions and only requires the
observed factual outcome. For this evaluation, predictions are made under each
target treatment sequence, and CRPS is computed against the observed factual
outcome only for held-out patients whose observed factual treatment sequence
exactly matches that target sequence. The CRPS results are reported in
Table~\ref{tab:crps_tau_real}.

\section{Implementation Details}
\label{app:implementation-details}

\subsection{Transformer history encoder}
\label{app:transformer-history-encoder}

\textbf{Architecture.}
All propensity, continuation, and second-stage target models use the same
Transformer specification to encode the observed pre-interventional history. We
follow the compact encoder-only design of \citet{frauen2025modelagnostic},
which was also adopted by the overlap-weighted learner of
\citet{hess2026overlapweighted}. At decision time $t$, each history token
concatenates the contemporaneous covariates and outcome with the preceding
treatment as $[X,V,Y,A_{\mathrm{prev}}]$, where the static covariates $V$ are repeated along the
time axis and $A_{-1}$ is represented by a zero sentinel. A linear input map
projects each token to the Transformer dimension, after which we add a fixed
sinusoidal positional encoding. A causal attention mask prevents look-ahead,
while a separate padding mask accommodates histories of different lengths.
The output at the last valid time point is passed through a linear layer and a
ReLU activation to obtain the context vector supplied to the propensity head
or conditional generator. The encoder is trained jointly with the head to
which it belongs; there is no separate representation-pretraining step. The
propensity nuisance is one time-shared model evaluated at each decision point,
whereas the stage-specific continuation models and the final target learners
have separate fitted parameters but use the same encoder specification.

\textbf{Hyperparameters.}
We fixed the complete Transformer profile in Table~\ref{tab:transformer-hparams}
before running the experiments. All encoder-specific values were inherited
from the longitudinal Transformer configurations of
\citet{frauen2025modelagnostic} and \citet{hess2026overlapweighted}. We kept
them fixed across datasets, treatment horizons, generative learners, and the
normalizing-flow and diffusion backends; none was tuned for the experiments
in this paper. This choice deliberately separates the comparison of causal
objectives from a search over history-encoder architectures.

\begin{table}[t]
\centering
\caption{Fixed hyperparameters of the Transformer history encoder. None of
these values was tuned specifically for our experiments.}
\label{tab:transformer-hparams}
\small
\setlength{\tabcolsep}{6pt}
\renewcommand{\arraystretch}{1.08}
\begin{tabularx}{\textwidth}{@{}l c X@{}}
\toprule
\textbf{Hyperparameter} & \textbf{Value} & \textbf{Role} \\
\midrule
Transformer blocks & $1$ & Causal encoder block \\
Model dimension $d_{\mathrm{model}}$ & $30$ & Token and attention representation width \\
Attention heads & $3$ & Head width $d_{\mathrm{model}}/3=10$ \\
Feed-forward dimension & $20$ & Hidden width inside the encoder block \\
Dropout probability & $0.1$ & Applied to the input embedding and encoder block \\
Readout dimension & $20$ & Dimension of the context passed to the downstream head \\
Activation & ReLU & Encoder feed-forward and readout nonlinearity \\
Normalization & Post-norm & Layer normalization after residual updates \\
Positional encoding & Fixed sinusoidal & Defined up to the padded history length \\
\bottomrule
\end{tabularx}
\end{table}

\subsection{Generative models}
\label{app:generative-model-implementation}

We implement the framework with conditional normalizing flows and diffusion
models. Their mathematical definitions and population objectives are given
separately in Appendix~\ref{app:generative-model-instantiations}; this subsection
records only the architectures, optimization choices, and hyperparameter
selection used in the experiments. Every generator receives the
20-dimensional context produced by the history encoder above. A continuation
model at relative step $\delta$ is conditioned on $\bar H_{t+\delta}$, whereas a target generator is
conditioned on $\bar H_t$. The intervention is fixed for each fitted model and
is therefore not included again in the conditioning input. All outcome
coordinates are represented on the standardized scale described in
Appendix~\ref{app:dataset-details}; any additional affine normalization internal
to a generator is inverted before returning samples on that evaluation scale.

\subsubsection{Normalizing-flow architecture}
Our implementation uses a standard Gaussian base, followed by a trainable
coordinate-wise affine map and one conditional monotone quadratic-spline
transform. For the scalar MIMIC-III outcome, a one-hidden-layer ELU network
maps the history context to the spline parameters. For the three-dimensional
synthetic outcome, the base is multivariate Gaussian and the final transform
is a conditional autoregressive spline whose one-hidden-layer autoregressive
network produces the parameters for all three coordinates. The latter is a
joint flow and hence models
contemporaneous dependence between outcome coordinates. The spline acts on
the standardized outcome scale. Its symmetric bound was fixed at $3$, whereas
the number of bins was tuned as described below. Optional Gaussian
perturbations of the context and outcome are used only as training
regularizers; validation NLLs and generated samples always use the
unperturbed model.

\subsubsection{Diffusion-model architecture}
We use a linear $\beta$ schedule. The diffusion step is represented by a fixed
sinusoidal embedding and concatenated with the noisy outcome. The denoiser is
a two-layer, per-observation MLP: a one-hidden-layer ELU hypernetwork maps the
history context to all weights and biases of this MLP, whose hidden width is
reported below and whose output has the same dimension as the outcome. Thus,
the complete denoising map adapts to the
patient history. Training samples a diffusion step uniformly, adds Gaussian
noise, and minimizes Eq.~\ref{eq:dm-denoising-loss}, averaged over outcome
coordinates. After fitting, samples are obtained from standard Gaussian noise
by ordinary ancestral DDPM sampling through all $S$ reverse transitions.
Because this model does not provide a tractable exact conditional density, we
use its denoising objective for validation.

\subsubsection{Training within the longitudinal learners}
The history encoder and its generative head are optimized jointly with Adam;
there is no separate encoder pretraining. Training pools all eligible rolling
origins, with continuation models indexed by relative regime stage and each
input restricted to the corresponding absolute-time history. At the terminal
step of generative recursive g-computation, the model is fitted to the observed terminal outcome among
factual rows whose last treatment agrees with the regime. Earlier continuation
models are then fitted backward to detached pseudo-outcomes drawn from the
already fitted, frozen successor model; these pseudo-outcomes are redrawn once
per epoch. We use 16 Monte Carlo draws for NFs and 64 for DMs. The IPTW, RA,
and \OrthoGen target generators are subsequently trained with their respective
objectives on the target-model split. For numerical stability, the cumulative inverse-propensity ratios used by IPTW and \textsc{OrthoGen} are clipped. Their validation objective is checked
after a 10-epoch warm-up, with patience 20 and minimum improvement $10^{-3}$;
the parameters with the best validation objective are restored. For RA and
\OrthoGen, an exponential moving average with decay $0.995$ is updated after each
optimizer step and used for validation and final sampling.

\subsubsection{Hyperparameter selection}
We tuned each generative family separately on each dataset by isolating the
terminal nuisance. Specifically, each candidate fits only the conditional law
of $Y_{t+\tau}$ given $\bar H_{t+\tau-1}$ on factual rows with
$A_{t+\tau-1}=a_{t+\tau-1}$. This is direct generative fitting of the outcome conditional on the
history available at the last treatment step: it does not require backward
pseudo-outcomes and does not fit a propensity model or a second-stage target
learner. Within the nuisance-estimation split, patients were divided into
80\% tuning-training and 20\% validation subsets. The same five replicate
datasets and patient splits were reused for every candidate, and candidates
were ranked by their mean validation generative loss. For NFs, this criterion
was the clean conditional NLL (the joint NLL for the three-dimensional
outcome). For DMs, it was the fixed-seed denoising loss in
Eq.~\ref{eq:dm-denoising-loss}. Neither test outcomes nor counterfactual oracle samples
entered model selection.

Table~\ref{tab:generative-search-space} gives one consolidated search space.
When the dataset-specific grids differed, we report the broader set containing
all candidate values considered across the datasets. Only parameters that were
actually varied are included. The NF spline bound was fixed at $3$. For
DMs, we fixed the sinusoidal step-embedding width at $16$, used a linear
schedule with $(\beta_1,\beta_S)=(10^{-4},0.02)$, a batch size of $64$, and a
maximum of $100$ epochs. Selection was nevertheless performed independently
for the semi-synthetic MIMIC-III, real-world MIMIC-III, and fully synthetic
datasets.

\begin{table}[t]
\centering
\caption{Consolidated hyperparameter search space for the terminal nuisance.
Only parameters that were varied are shown. $\sigma_C$ and $\sigma_Y$ denote
training-only Gaussian noise applied to the history context and standardized
outcome, respectively.}
\label{tab:generative-search-space}
\scriptsize
\setlength{\tabcolsep}{4pt}
\renewcommand{\arraystretch}{1.06}
\begin{tabularx}{\textwidth}{@{}llX@{}}
\toprule
\textbf{Family} & \textbf{Hyperparameter} & \textbf{Values considered} \\
\midrule
NF & Conditioner width & $\{20,32,64\}$ \\
   & Spline bins & $\{5,10,15,20\}$ \\
   & Learning rate & $\{3\!\times\!10^{-4},10^{-3}\}$ \\
   & Batch size & $\{32,64\}$ \\
   & Weight decay & $\{0,10^{-4}\}$ \\
   & Context noise $\sigma_C$ & $\{0,0.01,0.05\}$ \\
   & Outcome noise $\sigma_Y$ & $\{0,0.01,0.05\}$ \\
   & Maximum epochs & $\{30,50,100\}$ \\
\midrule
DM & Diffusion steps $S$ & $\{50,100,150\}$ \\
   & Denoiser hidden width & $\{10,20\}$ \\
   & Context noise $\sigma_C$ & $\{0,0.01\}$ \\
   & Learning rate & $\{3\!\times\!10^{-4},10^{-3},5\!\times\!10^{-3}\}$ \\
   & Weight decay & $\{0,10^{-4}\}$ \\
\bottomrule
\end{tabularx}
\end{table}

The selected configurations are summarized in
Table~\ref{tab:generative-selected-hparams}. The selected generator
architecture was used for all continuation and target models of the
corresponding backend and dataset. The listed optimization values are those of
the terminal nuisance.

\begin{table}[t]
\centering
\caption{Selected generative-model hyperparameters for each dataset.}
\label{tab:generative-selected-hparams}
\scriptsize
\setlength{\tabcolsep}{4pt}
\renewcommand{\arraystretch}{1.08}
\begin{tabularx}{\textwidth}{@{}llXX@{}}
\toprule
\textbf{Dataset} & \textbf{Family} & \textbf{Selected architecture} & \textbf{Selected optimization and regularization} \\
\midrule
Semi-synthetic MIMIC-III
& NF
& Width $20$; $10$ bins
& Learning rate $10^{-3}$; batch $64$; weight decay $10^{-4}$; $100$ epochs; $(\sigma_C,\sigma_Y)=(0.01,0)$ \\
& DM
& $S=100$; denoiser width $20$
& Learning rate $3\!\times\!10^{-4}$; weight decay $10^{-4}$; $\sigma_C=0.01$ \\
\midrule
Real-world MIMIC-III
& NF
& Width $20$; $10$ bins
& Learning rate $10^{-3}$; batch $64$; weight decay $10^{-4}$; $100$ epochs; $(\sigma_C,\sigma_Y)=(0.01,0)$ \\
& DM
& $S=100$; denoiser width $20$
& Learning rate $3\!\times\!10^{-4}$; weight decay $10^{-4}$; $\sigma_C=0.01$ \\
\midrule
Fully synthetic
& NF
& Width $20$; $10$ bins
& Learning rate $10^{-3}$; batch $64$; weight decay $10^{-4}$; $30$ epochs; $(\sigma_C,\sigma_Y)=(0,0)$ \\
& DM
& $S=150$; denoiser width $20$
& Learning rate $3\!\times\!10^{-4}$; weight decay $10^{-4}$; $\sigma_C=0$ \\
\bottomrule
\end{tabularx}
\end{table}

\section{Computational runtime}
\label{app:runtime}

We report wall-clock runtimes to provide a representative measure
of the computational cost of our experiments. Table~\ref{tab:runtime} reports
the average end-to-end runtime of one replicate for the intermediate treatment
horizon $\tau=3$, using the same configurations as in the main experiments.
Each replicate includes data preprocessing, nuisance estimation, fitting of all
four learners (PI, IPTW, RA, and \OrthoGen), and evaluation. The reported
values are averages over five replicates and are rounded to the nearest minute.
All experiments use $n_{\mathrm{train}}=3000$ and were run on an Amazon
SageMaker \texttt{ml.g4dn.xlarge} instance equipped with one NVIDIA T4 GPU
(16\,GB GPU memory), 4 vCPUs, and 16\,GiB of system memory.

\begin{table}[t]
\centering
\caption{Average end-to-end wall-clock runtime per replicate for a
representative treatment horizon $\tau=3$. Each replicate includes nuisance
estimation, fitting of PI, IPTW, RA, and \OrthoGen, and evaluation.}
\label{tab:runtime}
\small
\setlength{\tabcolsep}{7pt}
\renewcommand{\arraystretch}{1.08}
\begin{tabular}{lcc}
\toprule
\textbf{Dataset} & \textbf{NFs} & \textbf{DMs} \\
\midrule
Fully synthetic          & $56$ min  & $122$ min \\
Semi-synthetic MIMIC-III & $110$ min & $176$ min \\
Real-world MIMIC-III     & $27$ min  & $63$ min \\
\bottomrule
\end{tabular}
\end{table}

The computational requirements of the individual learners differ according to
their adjustment strategy. PI requires fitting the sequence of $\tau$
continuation models, whereas IPTW requires one time-shared propensity model
and one target generator. RA adds a target generator on top of the continuation
sequence, while \OrthoGen combines the time-shared propensity model and
continuation sequence with a second-stage target generator. In our
implementation, nuisance estimates are shared across the corresponding
learners; hence, we report the end-to-end runtime of the complete experiment. As expected, the DM instantiations are computationally more demanding
than the NF instantiations due to their iterative denoising procedure.

\end{document}